%% file: main.tex
\documentclass{article}
\PassOptionsToPackage{numbers, compress}{natbib}

\usepackage[final]{neurips_2019}

\usepackage[utf8]{inputenc} 
\usepackage[T1]{fontenc}    
\usepackage{hyperref}       
\usepackage{url}            
\usepackage{booktabs}       
\usepackage{amsfonts}       
\usepackage{nicefrac}       
\usepackage{microtype}      
\usepackage{amsmath}
\usepackage{amsthm}
\usepackage{todonotes}
\newtheorem{theorem}{Theorem}
\newtheorem{prop}{Proposition}
\usepackage{multirow}

\input{macros}

\title{Recurrent Kernel Networks}

\author{%
  Dexiong Chen \\
  Inria\thanks{Univ. Grenoble Alpes, Inria, CNRS, Grenoble INP, LJK, 38000 Grenoble, France.} \\
  \texttt{dexiong.chen@inria.fr} \\
  \And
  Laurent Jacob \\
  CNRS\thanks{Univ. Lyon, Universit\'e Lyon 1, CNRS, Laboratoire de Biom\'etrie et Biologie Evolutive UMR 5558, 69000 Lyon, France} \\
  \texttt{laurent.jacob@univ-lyon1.fr} \\
  \And
  Julien Mairal \\
  Inria\footnotemark[1] \\
  \texttt{julien.mairal@inria.fr} \\
}

\begin{document}

\maketitle

\begin{abstract}
\input{abstract.tex}

\end{abstract}

\section{Introduction}
\input{introduction}

\section{Background on Kernel Methods and String Kernels}\label{sec:background}
\input{background}

\section{Recurrent Kernel Networks}\label{sec:rkn}
\input{method}

\section{Experiments}\label{sec:exp}
\input{experiments}

\subsubsection*{Acknowledgments}
 We thank the anonymous reviewers for their insightful comments and suggestions. This work has been supported by the grants from ANR (FAST-BIG project ANR-17-CE23-0011-01), by the ERC grant number 714381 (SOLARIS), and ANR 3IA
MIAI@Grenoble Alpes.

\bibliographystyle{abbrv}
\bibliography{mybib}

\newpage

\appendix
\input{appendix}
\end{document}

%% file: macros.tex
\def\x{{\mathbf x}}

\def\i{{\mathbf i}}
\def\j{{\mathbf j}}
\def\c{{\mathbf c}}

\def\h{{\mathbf h}}
\def\w{{\mathbf w}}

\def\R{{\mathbb R}}
\def\S{{\mathcal S}}
\def\I{{\mathcal I}}
\def\K{{\mathcal K}}
\def\A{{\mathcal A}}
\def\H{{\mathcal H}}
\def\Xcal{{\mathcal X}}
\def\Ycal{{\mathcal Y}}
\def\Hcal{{\mathcal H}}
\def\Real{{\mathbb R}}
\def\Ecal{{\mathcal E}}
\def\z{{\mathbf z}}
\def\gap{{\text{gaps}}}
\def\defin{:=}
\def\b{{\mathbf b}}
\def\eg{{\it{e.g.}}}
\def\u{{\mathbf u}}
\def\gmp{{\text{gmp}}}
\DeclareMathOperator*{\argmin}{arg\,min}

%% file: abstract.tex
Substring kernels are classical tools for representing biological sequences or
text. However, when large amounts of annotated data are available, models
that allow end-to-end training such as neural networks are often preferred.
Links between recurrent neural networks (RNNs) and substring kernels
have recently been drawn, by formally
showing that RNNs with specific activation functions 
were points in a reproducing kernel Hilbert space (RKHS).
In this paper, we revisit this link by generalizing convolutional
kernel networks---originally related to a relaxation of the mismatch
kernel---to model gaps in sequences. 
It results in a new type of
recurrent neural network 
which can be trained end-to-end with backpropagation, or without supervision
by using kernel approximation techniques.
We experimentally show that our approach is well suited to biological
sequences, where it outperforms existing methods for protein classification
tasks.

%% file: introduction.tex
Learning from biological sequences is important for a variety of scientific fields such as evolution~\cite{flagel} or human
health~\cite{Topol2019High}. In order to use classical statistical models, 
a first step is often to map sequences to vectors of fixed size,
while retaining relevant features for the considered learning task. For a long
time, such features have been extracted from sequence alignment, either against a reference
or between each others~\cite{Auton2015AGR}.
The
resulting features are appropriate for sequences that are similar
enough, but they become ill-defined when sequences are not suited to
alignment. This includes important cases such as microbial genomes,
distant species, or human diseases, and calls for alternative representations~\citep{Computational2016Computational}.

String kernels provide generic representations for biological
sequences, most of which do not require global alignment
~\citep{scholkopf2004kernel}. In particular, a classical approach
maps sequences to a huge-dimensional feature space by enumerating
statistics about all occuring subsequences. 
These subsequences may be simple classical $k$-mers leading to the
spectrum kernel~\cite{leslie2001spectrum}, $k$-mers up to
mismatches~\cite{leslie2004mismatch}, or gap-allowing
subsequences~\cite{lodhi2002text}.  Other approaches involve kernels
based on a generative
model~\cite{jaakkola1999using,tsuda2002marginalized}, or based on
local alignments between sequences~\cite{vert2004convolution} inspired
by convolution kernels~\cite{haussler1999convolution,watkins1999dynamic}.

The goal of kernel design is then to encode prior knowledge in the learning process. For instance, modeling gaps in biological
sequences is important since it allows taking into account short
insertion and deletion events, a common source of genetic
variation.
However, even though kernel methods are good at encoding prior knowledge,
they provide fixed task-independent representations.
When large
amounts of data are available, approaches that optimize the data
representation for the prediction task are now often preferred. 
For instance, convolutional
neural networks~\cite{lecun1989backpropagation} are commonly used for
DNA sequence
modeling~\cite{alipanahi2015predicting,angermueller2016deep,Zhou2015Predicting},
and have been successful for natural language
processing~\cite{kalchbrenner2014convolutional}. While convolution
filters learned over images are interpreted as image patches, those
learned over sequences are viewed as sequence
motifs. RNNs such as long short-term memory
networks (LSTMs)~\cite{hochreiter1997long} are also commonly used in
both biological~\cite{hochreiter2007fast} and natural language
processing contexts~\citep{cho2014learning,merity2017regularizing}. 

Motivated by the regularization mechanisms of kernel methods, which
are useful when the amount of data is small and are yet imperfect in
neural networks, hybrid approaches have been developed between the
kernel and neural networks
paradigms~\cite{cho2009kernel,morrow2017convolutional,zhang2016convexified}.
Closely related to our work, the convolutional kernel network (CKN)
model originally developed for images~\cite{mairal2016end} was
successfully adapted to biological sequences
in~\cite{chen2019bio}. CKNs for sequences consist in a continuous
relaxation of the mismatch kernel: while the latter represents a
sequence by its content in $k$-mers up to a few discrete errors, the
former considers a continuous relaxation, leading
to an infinite-dimensional sequence representation. 
Finally, a kernel approximation relying on the Nystr\"om
method~\cite{williams2001using} projects the mapped sequences to a
linear subspace of the RKHS, spanned by a finite
number of motifs. When these motifs are learned end-to-end with
backpropagation, learning with CKNs can also be thought of as
performing feature selection in the---infinite dimensional---RKHS.

In this paper, we generalize CKNs for sequences by allowing gaps in
motifs, motivated by genomics applications. The kernel map retains the
convolutional structure of CKNs but the kernel approximation that we
introduce can be computed using a recurrent network, which we call
recurrent kernel network (RKN). This RNN arises from the dynamic
programming structure used to compute efficiently the substring kernel
of~\cite{lodhi2002text}, a link already exploited
by~\cite{lei2017deriving} to derive their sequence neural
network, which was a source of inspiration for our work. Both our kernels rely on a RNN to build a representation of
an input sequence by computing a string kernel between this sequence
and a set of learnable filters. Yet, our model exhibits several
differences with~\cite{lei2017deriving}, who use the regular substring
kernel of~\cite{lodhi2002text} and compose this representation with
another non-linear map---by applying an activation function to the
output of the RNN. By contrast, we obtain a different RKHS directly by relaxing
the substring kernel to allow for inexact matching at the compared positions,
and embed the Nystr\"om approximation within the RNN.
The resulting feature space can be interpreted as a
continuous neighborhood around all substrings (with gaps) of the described
sequence. Furthermore, our RNN provides a finite-dimensional approximation of
the relaxed kernel, relying on the Nystr\"om approximation
method~\cite{williams2001using}. As a consequence, RKNs may be learned
in an unsupervised manner (in such a case, the goal is to approximate
the kernel map), and with supervision with backpropagation, which may be interpreted as performing
feature selection in the RKHS.

\paragraph{Contributions.}
In this paper, we make the following contributions:\\
~$\bullet$~We generalize convolutional kernel networks for
sequences~\cite{chen2019bio} to allow gaps, an important option for biological data. 
As in~\cite{chen2019bio}, we observe that the kernel formulation brings practical benefits over traditional CNNs or RNNs~\cite{hochreiter2007fast} when the amount of labeled data is small or moderate.\\
~$\bullet$~We provide a kernel point of view on recurrent neural
networks with new unsupervised and supervised learning algorithms. The resulting
feature map can be interpreted in terms of gappy motifs, and end-to-end
learning amounts to performing feature selection.\\
~$\bullet$~Based on~\cite{murray2014generalized}, we propose a new way to simulate max pooling in RKHSs, thus solving a
classical discrepancy between theory and practice in the literature of string
kernels, where sums are often replaced by a maximum operator that does not
ensure positive definiteness~\cite{vert2004convolution}.

%% file: background.tex
Kernel methods consist in mapping data points living in a set $\Xcal$ to a
possibly infinite-dimensional Hilbert space~$\Hcal$, through a mapping function $\Phi: \Xcal \to \Hcal$, before learning
a simple predictive model in~$\Hcal$~\citep{scholkopf2002learning}.
The so-called kernel trick allows to perform learning without explicitly computing
this mapping, as long as the inner-product~$K(\x,\x')=\langle \Phi(\x), \Phi(\x')\rangle_\Hcal$ between two points~$\x, \x'$ can be 
efficiently computed. Whereas kernel methods traditionally lack scalability since they require computing an
$n\times n$ Gram matrix, where $n$ is the amount of training data,
recent approaches based on approximations have managed to make kernel methods work at large scale in many cases~\cite{rahimi2008random,williams2001using}.

For sequences in $\Xcal=\A^*$, which is the
set of sequences of any possible length over an alphabet~$\A$, the mapping $\Phi$
often enumerates subsequence content. For instance, the spectrum kernel 
maps sequences to a fixed-length vector 
$\Phi(\x) = \left(\phi_u(\x)\right)_{u\in\A^k}$, where
$\A^k$ is the set of $k$-mers---length-$k$ sequence of characters in
$\A$ for some $k$ in $\mathbb{N}$, and $\phi_u(\x)$ counts the number of occurrences of $u$ in~$\x$~\citep{leslie2001spectrum}.
The mismatch kernel~\citep{leslie2004mismatch} operates similarly, but $\phi_u(\x)$ counts the occurrences of $u$ up to a few mismatched letters, which 
is useful when $k$ is large and exact occurrences are rare.

\subsection{Substring kernels}
As~\cite{lei2017deriving}, we consider the substring kernel introduced
in~\cite{lodhi2002text}, which allows to model the presence of gaps
when trying to match a substring $u$ to a sequence $\x$. 
Modeling gaps requires introducing the following notation:
$\I_{\x,k}$ denotes the set of indices of sequence $\x$ with
$k$ elements $(i_1,\dots,i_k)$ satisfying
$1\le i_1<\cdots< i_k\le |\x|$, where $|\x|$ is the length of $\x$.
For an index set $\i$ in $\I_{\x,k}$, we may now consider the subsequence
$\x_{\i} = (\x_{i_1},\dots,\x_{i_k})$ of $\x$ indexed by $\i$.
Then, the substring kernel takes the same form as the mismatch and spectrum kernels, but 
$\phi_u(\x)$ counts all---consecutive or not---subsequences of $\x$ equal to
$u$, and weights them by the number of gaps.
Formally, we consider a parameter $\lambda$ in $[0,1]$, and 
$\phi_u(\x) = \sum_{\i\in\I_{\x,k}} \lambda^{\gap(\i)}\delta(u, \x_{\i})$, where $\delta(u, v)=1$ if and only if $u=v$, and
$0$ otherwise, and $\gap(\i) \defin i_k - i_1 - k +
  1$ is the number
of gaps in the index set $\i$.
When $\lambda$ is small, gaps are heavily penalized, whereas a value
close to~$1$ gives similar
weights to all occurrences. Ultimately, the resulting kernel between two sequences $\x$ and $\x'$ is
\begin{equation}\label{eq:substring_kernel}
  \K^s(\x,\x'):=\sum_{\i\in\I_{\x,k}} \sum_{\j\in\I_{\x',k}}
  \lambda^{\gap(\i)}\lambda^{\gap(\j)} \delta\left(\x_{\i},
    \x'_{\j}\right).
\end{equation}
As we will see in Section~\ref{sec:rkn}, our RKN model relies
on~(\ref{eq:substring_kernel}), but unlike~\cite{lei2017deriving}, we replace
the quantity $\delta(\x_{\i}, \x'_{\j})$ that matches exact
occurrences by a relaxation, allowing more subtle comparisons. Then, we will show that the 
model can be interpreted as a gap-allowed extension of CKNs for
sequences.
We also note that even though~$\K^s$ seems computationally expensive at first
sight, it was shown in~\cite{lodhi2002text} that~(\ref{eq:substring_kernel}) admits a dynamic
programming structure leading to efficient computations.

\subsection{The Nystr\"om method}
\label{sec:nystrom}
When computing the  Gram
matrix is infeasible, it is typical to use kernel approximations 
~\cite{rahimi2008random,williams2001using}, consisting 
 in finding a $q$-dimensional mapping $\psi: \Xcal \to
\Real^q$ such that the kernel $K(\x,\x')$ can be approximated by a Euclidean
inner-product $\langle \psi(\x), \psi(\x')\rangle_{\R^q}$.
Then, kernel methods can be simulated by a linear model operating on $\psi(\x)$,
which does not raise scalability issues if $q$ is reasonably small.
Among kernel approximations, the Nystr\"om method consists in
projecting points of the RKHS onto a $q$-dimensional subspace, allowing to
represent points into a $q$-dimensional coordinate system. 

Specifically, consider a collection of $Z=\{\z_1,\ldots,\z_q\}$ points in $\Xcal$ and
consider the subspace 
$$\Ecal = \text{Span}(\Phi(\z_1), \ldots, \Phi(\z_q))~~~\text{and define}~~~\psi(\x) = K_{ZZ}^{-\frac{1}{2}}K_Z(\x),$$
where $K_{ZZ}$ is the $q \times q$ Gram matrix of $K$ restricted to the 
samples $\z_1, \ldots, \z_q$ and $K_Z(\x)$ in~$\R^q$
carries the kernel values $K(\x, \z_j), j=1, \ldots, q$. 
This approximation
only requires $q$ kernel evaluations and often retains good
performance for learning. Interestingly as noted
in~\cite{mairal2016end}, $\langle \psi(\x), \psi(\x')\rangle_{\R^q}$ is
exactly the inner-product in $\H$ between the projections of $\Phi(\x)$
and $\Phi(\x')$ onto~$\Ecal$,
which remain in~$\Hcal$. 

When $\Xcal$ is a Euclidean space---this can be the case for sequences when
using a one-hot encoding representation, as discussed later--- 
a good set of anchor points $\z_j$ can be obtained by simply clustering the
data and choosing the centroids as anchor points~\cite{zhang2008improved}.
The goal is then to obtain a subspace $\Ecal$ that spans data as best as possible.
Otherwise, 
previous works on kernel networks~\cite{chen2019bio,mairal2016end} have also
developed procedures to learn the set of anchor points end-to-end by
optimizing over the learning objective. This approach can then be seen as performing 
feature selection in the RKHS.

%% file: method.tex
With the previous tools in hand, we now introduce RKNs. We show that
it admits variants of CKNs, substring and local
alignment kernels as special cases, and we discuss its relation
with RNNs.

\subsection{A continuous relaxation of the substring kernel allowing mismatches}
From now on, and with an abuse of notation, we represent characters in
$\A$ as vectors in $\R^d$. For instance, when using one-hot encoding,
a DNA sequence $\x=(\x_1,\ldots,\x_m)$ of length $m$ can be seen as a
4-dimensional sequence where each $\x_j$ in $\{0, 1\}^4$ has a unique
non-zero entry indicating which of $\{A,C,G,T\}$ is present at the
$j$-th position, and we denote by $\Xcal$ the set of such sequences.
We now define the single-layer RKN as a generalized substring kernel~\eqref{eq:substring_kernel}
in which the indicator function $\delta(\x_\i,\x'_\j)$ is replaced
by a kernel for $k$-mers:
\begin{equation}
	\K_k(\x,\x'):=\sum_{\i\in\I_{\x,k}} \sum_{\j\in\I_{\x',k}}
        \lambda_{\x,\i}\lambda_{\x,\j} e^{-\frac{\alpha}{2}\|\x_\i - \x'_\j\|^2},\label{eq:kRKN}
\end{equation}
where we assume that the vectors representing characters have unit
$\ell_2$-norm, such that $e^{-\frac{\alpha}{2}\|\x_\i -
  \x'_\j\|^2}=e^{\alpha( \langle \x_\i, \x'_\j \rangle -
  k)}=\prod_{t=1}^k e^{\alpha\left(\langle \x_{i_t},\x'_{j_t}\rangle -1\right)}$
is a dot-product kernel, and $\lambda_{\x,\i}=\lambda^{\gap(\i)}$ if
we follow~(\ref{eq:substring_kernel}). 

For $\lambda=0$ and using the
convention $0^0=1$, all the terms in these sums are zero except those
for $k$-mers with no gap, and we recover the kernel of the CKN model
of~\cite{chen2019bio} with a convolutional structure---up to the
normalization, which is done $k$-mer-wise in CKN instead of
position-wise.

Compared to~\eqref{eq:substring_kernel}, the relaxed
version~\eqref{eq:kRKN} accommodates inexact $k$-mer matching. This is
important for protein sequences, where it is common to consider different similarities between
amino acids in terms of substitution
frequency along evolution~\citep{henikoff1992amino}. This is also
reflected in the underlying sequence representation in the RKHS illustrated in Figure~\ref{fig:rkhs}: by considering $\varphi(.)$ the
kernel mapping and RKHS~$\Hcal$ such that
$K(\x_\i, \x'_\j) = e^{-\frac{\alpha}{2}\|\x_\i - \x'_\j\|^2} = \langle \varphi(\x_\i),
\varphi(\x'_\j) \rangle_{\Hcal}$, we have
\begin{equation}
	\K_k(\x,\x')= \left\langle \sum_{\i\in\I_{\x,k}} \lambda_{\x,\i} \varphi(\x_\i)  , \sum_{\j\in\I_{\x',k}} \lambda_{\x,\j} \varphi(\x'_\j)\right\rangle_{\Hcal}. \label{eq:kRKNconv}
\end{equation}
A natural feature map for a sequence $\x$ is therefore
$\Phi_k(\x) = \sum_{\i\in\I_{\x,k}} \lambda_{\x,\i} \varphi(\x_\i)$: using the
RKN amounts to representing~$\x$ by a mixture of continuous
neighborhoods $\varphi(\x_\i): \z \mapsto e^{-\frac{\alpha}{2}\|\x_\i - \z\|^2}$ centered on all
its $k$-subsequences $\x_\i$ , each weighted by the corresponding
$\lambda_{\x,\i}$ (\eg, $\lambda_{\x,\i} = \lambda^{\gap(\i)}$). As a particular 
case, a feature map of CKN~\cite{chen2019bio} is the sum of the kernel mapping 
of all the $k$-mers without gap.

\begin{figure}[hbtp]
\centering
   \includegraphics[width=0.9\linewidth]{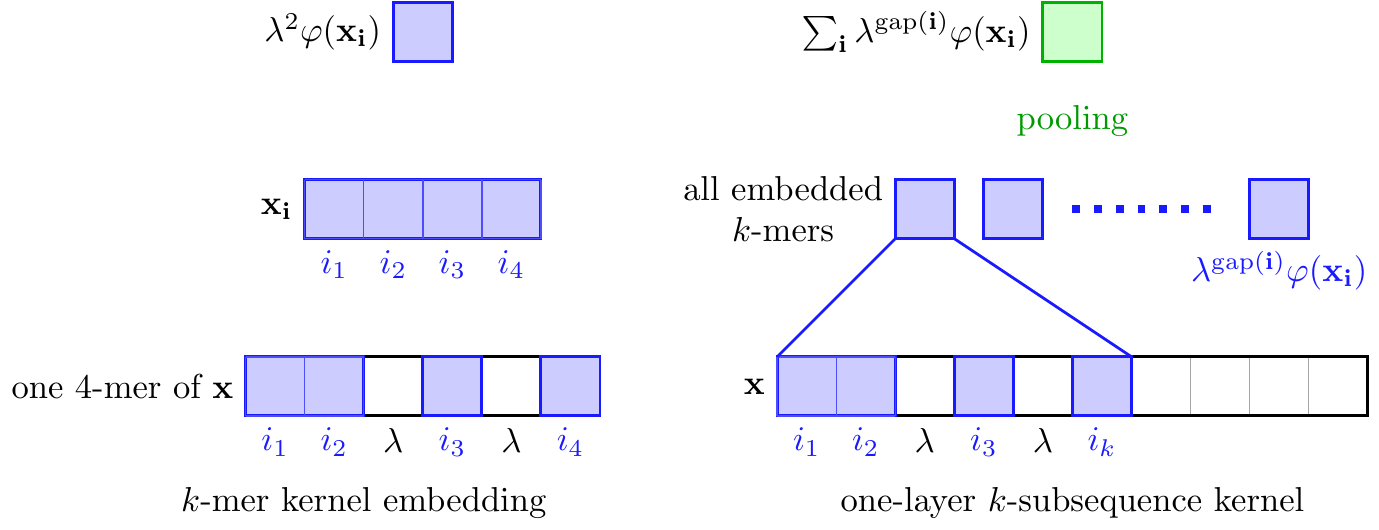}
   \caption{Representation of sequences in a RKHS based on~$\K_k$ with $k=4$ and $\lambda_{\x,\i}=\lambda^{\gap(\i)}$.}\label{fig:rkhs}
\end{figure}

\subsection{Extension to all $k$-mers and relation to the local alignment kernel}
Dependency in the hyperparameter $k$ can be removed by summing $\K_k$
over all possible values:
\begin{equation*}
	\K_{\textrm{sum}}(\x,\x'):= \sum_{k=1}^\infty\K_k(\x,\x') =
        \sum_{k=1}^{\max(|\x|, |\x'|)}\K_k(\x,\x').
\end{equation*}
Interestingly, we note that $\K_\text{sum}$
admits the local alignment kernel
of~\cite{vert2004convolution} as a special case.
More precisely, local alignments are defined via the
tensor product set $\A_k(\x,\x'):=\I_{\x,k}\times\I_{\x',k}$, which
contains all possible alignments of $k$ positions between a pair of
sequences $(\x, \x')$. The local alignment score of each such
alignment $\pi=(\i,\j)$ in $\A_k(\x,\x')$ is defined, by
\cite{vert2004convolution}, as $S(\x,\x', \pi):=\sum_{t=1}^{k}
s(\x_{i_t}, \x'_{j_t})-\sum_{t=1}^{k-1}[g(i_{t+1} - i_t - 1)+g(j_{t+1}
  - j_t - 1)]$, where $s$ is a symmetric substitution function and $g$
is a gap penalty function.  The local alignment kernel in
\cite{vert2004convolution} can then be expressed in terms of the above
local alignment scores (Thrm. 1.7 in \cite{vert2004convolution}):
\begin{equation}
  \label{eq:LA}
	K_{LA}(\x,\x')=\sum_{k=1}^{\infty}K_{LA}^{k}(\x,\x'):=\sum_{k=1}^{\infty}\sum_{\pi\in\A_k(\x,\x')}\exp(\beta \S(\x,\x',\pi))~~~\text{for some}~\beta > 0.
\end{equation}
When the gap penalty function is linear---that is, $g(x)=cx$ with
$c>0$, $K_{LA}^{k}$ becomes
$K_{LA}^{k}(\x,\x')=\sum_{\pi\in\A_k(\x,\x')}\exp(\beta
\S(\x,\x',\pi)) =\sum_{(\i,\j)\in\A_k(\x,\x')} e^{-c\beta \gap(\i)}
e^{-c\beta \gap(\j)} \prod_{t=1}^k e^{\beta s
  (\x_{i_t},\x'_{j_t})}$. 
When $s(\x_{i_t},\x'_{j_t})$ can be written as an inner-product $\langle \psi_s(\x_{i_t}),\psi_s(\x'_{j_t})\rangle$ between normalized vectors,
we see that $K_{LA}$ becomes a special case of~\eqref{eq:kRKN}---up to a constant
factor---with $\lambda_{\x,\i}=e^{-c\beta\gap(\i)}$,
$\alpha={\beta }$. 

This observation sheds new lights on the relation between the substring and
local alignment kernels, which will inspire new algorithms in the sequel.  To the best of our knowledge, the link we will
provide between RNNs and local alignment kernels is also new.

\subsection{Nystr\"om approximation and recurrent neural networks}\label{subsec:nystrom}
As in CKNs, we now use the Nystr\"om approximation method as a
building block to make the above kernels tractable. According
to~(\ref{eq:kRKNconv}), we may first use the Nystr\"om method
described in Section~\ref{sec:nystrom} to find an approximate
embedding for the quantities $\varphi(\x_\i)$, where $\x_\i$ is one of
the $k$-mers represented as a matrix in $\Real^{k \times d}$. This is
achieved by choosing a set $Z=\{\z_1,\ldots,\z_q\}$ of anchor points
in $\Real^{k \times d}$, and by encoding $\varphi(\x_\i)$ as
$K_{ZZ}^{-{1}/{2}}K_Z(\x_\i)$---where $K$ is the kernel of
$\Hcal$. Such an approximation for $k$-mers yields the $q$-dimensional
embedding for the sequence $\x$:
\begin{equation}
    \psi_k(\x) = \sum_{\i\in\I_{\x,k}} \lambda_{\x,\i} K_{ZZ}^{-\frac{1}{2}} K_Z(\x_\i) =  K_{ZZ}^{-\frac{1}{2}} \sum_{\i\in\I_{\x,k}} \lambda_{\x,\i}K_Z(\x_\i).\label{eq:psik}
\end{equation}
Then, an approximate feature map~$\psi_{\text{sum}}(\x)$ for the kernel $\K_\text{sum}$ can be obtained by concatenating the embeddings $\psi_1(\x),\ldots,\psi_k(\x)$ for $k$ large enough.

\paragraph{The anchor points as motifs.}
The continuous relaxation of the substring kernel presented in (\ref{eq:kRKN}) allows us to 
learn anchor points that can be interpreted as sequence motifs, 
where each position can encode a mixture of letters. This can lead to more
relevant representations than $k$-mers for learning on biological sequences.
For example, the fact that a DNA sequence is bound by a particular
transcription factor can be associated with the presence of a T followed by
either a G or an A, followed by another T, would require two $k$-mers but
a single motif \cite{chen2019bio}. Our kernel is able to perform
such a comparison.

\paragraph{Efficient computations of $\K_k$ and $\K_\text{sum}$ approximation via RNNs.}
A naive computation of $\psi_k(\x)$ would require enumerating all
substrings present in the sequence, which may be exponentially large
when allowing gaps. For this reason, we use the classical dynamic
programming approach of substring kernels~\cite{lei2017deriving,lodhi2002text}.
Consider then the computation of $\psi_j(\x)$ defined in~(\ref{eq:psik}) 
for $j=1,\ldots,k$ as well as a set of anchor points $Z_k=\{\z_1,\ldots,\z_q\}$ with the $\z_i$'s in $\Real^{d \times k}$.
We also denote by~$Z_j$ the set obtained when keeping only $j$-th first positions (columns) of the $\z_j$'s, leading to 
$Z_j=\{[\z_1]_{1:j},\ldots,[\z_q]_{1:j}\}$, which will serve as anchor points for the kernel $\K_j$ to compute $\psi_j(\x)$.
Finally, we denote by $\z_i^j$ in~$\Real^d$ the $j$-th column of $\z_i$ such that
$\z_i=[\z_i^1,\ldots,\z_i^k]$.
Then, the embeddings~$\psi_1(\x),\ldots,\psi_k(\x)$ can be computed recursively by using the following theorem:
\begin{theorem}
  \label{thm:kzx}
  For any $j\in\{1,\dots,k\}$ and
  $t\in\{1,\dots,|\x|\}$,
  \begin{equation}
   \psi_j(\x_{1:t}) =K_{Z_j Z_j}^{-\frac{1}{2}}
\begin{cases}
     \c_j[t] & \text{if } \lambda_{\x,\i}=\lambda^{|\x|-i_1-j+1}, \\
     \h_j[t] & \text{if } \lambda_{\x,\i}=\lambda^{\gap(\i)},
    \end{cases}
  \end{equation}
 where $c_j[t]$ and $h_j[t]$ form a sequence of vectors in $\Real^q$ indexed by $t$ such that $c_j[0]=h_j[0]=0$, and $c_0[t]$ is a vector that contains only ones, while the sequence obeys the recursion
  \begin{equation}\label{eqn:rkn}
    \begin{aligned}
      \c_j[t]&=\lambda\c_j[t-1]+\c_{j-1}[t-1] \odot \b_j[t] & 1\le j\le k,\\
      \h_j[t]&=\h_j[t-1]+\c_{j-1}[t-1]\odot \b_j[t]  & 1\le j\le k,
    \end{aligned}
  \end{equation}
  where $\odot$ is the elementwise multiplication operator and $\b_j[t]$ is a vector in $\Real^q$ whose entry~$i$ in $\{1,\ldots,q\}$ is $e^{-\frac{\alpha}{2}\|\x_t-\z_j^i\|^2} = e^{\alpha (\langle \x_t, \z_j^i\rangle -1)}$ and $\x_t$ is the $t$-th character of $\x$.
\end{theorem}
A proof is provided in Appendix~\ref{appendix:nystrom} and is based on classical recursions
for computing the substring kernel, which were interpreted as RNNs
by~\cite{lei2017deriving}.  The main difference in the RNN structure
we obtain is that their non-linearity is applied over the outcome of
the network, leading to a feature map formed by composing the feature
map of the substring kernel of~\cite{lodhi2002text} and another one
from a RKHS that contains their non-linearity. By contrast, our
non-linearities are built explicitly in the substring kernel, by
relaxing the indicator function used to compare characters. The
resulting feature map is a continuous neighborhood around all
substrings of the described sequence.
In addition, the Nystr\"om method yields an orthogonalization
factor $K_{ZZ}^{-{1}/{2}}$ to the output $K_Z(\x)$ of the network to
compute our approximation, which is perhaps the only non-standard component of our RNN.
This factor provides
an interpretation of $\psi(\x)$ as a kernel approximation. As discussed next, 
it makes it possible to learn the anchor points by $k$-means, see~\cite{chen2019bio}, 
which also makes the initialization of the supervised learning procedure
simple  without having to 
deal with the scaling of the initial motifs/filters $\z_j$.

\paragraph{Learning the anchor points~$Z$.}
We now turn to the application of RKNs to supervised
learning. Given $n$ sequences $\x^1,\dots,\x^n$ in $\Xcal$ and their
associated labels $y^1,\dots,y^n$ in $\Ycal$, \emph{e.g.}, $\Ycal=\{-1,1\}$
for binary classification or $\Ycal=\R$ for regression, our objective
is to learn a function in the RKHS $\Hcal$ of $\K_k$ by minimizing
\begin{equation*}
	\min_{f\in\Hcal} \frac{1}{n}\sum_{i=1}^n L(f(\x^i),y^i)+\frac{\mu}{2} \|f\|^2_{\Hcal},
\end{equation*}
where $L:\R\times\R\to\R$ is a convex loss function that measures the
fitness of a prediction $f(\x^i)$ to the true label $y^i$ and
$\mu$ controls the smoothness of the predictive function. After
injecting our kernel approximation
$\K_k(\x,\x')\simeq\langle \psi_k(\x), \psi_k(\x') \rangle_{\R^q}$,
the problem becomes
\begin{equation}\label{eq:sup}
	\min_{\w\in\R^q} \frac{1}{n}\sum_{i=1}^n L\left(\langle \psi_k(\x^i),\w\rangle ,y^i\right) + \frac{\mu}{2}\|\w\|^2.
\end{equation}

Following~\cite{chen2019bio,mairal2016end}, we can learn the
anchor points $Z$ without exploiting training labels, by applying a
$k$-means algorithm to all (or a subset of) the $k$-mers extracted
from the database and using the obtained centroids as anchor points.
Importantly, once $Z$ has been obtained, the linear function
parametrized by $\w$ is still optimized with respect to the supervised
objective~\eqref{eq:sup}. This procedure can be thought of as learning
a general representation of the sequences disregarding the supervised
task, which can lead to a relevant description while limiting
overfitting. 

Another strategy consists in optimizing~(\ref{eq:sup}) jointly over
$(Z,\w)$, after observing that
$\psi_k(\x) = K_{ZZ}^{-{1}/{2}} \sum_{\i\in\I_{\x,k}}
\lambda_{\x,\i}K_Z(\x_\i)$ is a smooth function of $Z$. 
Learning can be achieved by using backpropagation over $(Z,\w)$, or by
using an alternating minimization strategy between $Z$ and $\w$. It
leads to an end-to-end scheme where both the representation and the
function defined over this representation are learned with respect to
the supervised objective~\eqref{eq:sup}.  
Backpropagation rules for most operations are classical, except for the matrix inverse square root
function, which is detailed in Appendix~\ref{appendix:backprop}.  Initialization is also
parameter-free since the unsupervised learning approach may be used
for that.

\subsection{Extensions}\label{subsec:ext}

\paragraph{Multilayer construction.}

In order to account for
long-range dependencies, it is possible to construct a multilayer model
based on kernel compositions similar to~\cite{lei2017deriving}. Assume that $\K^{(n)}_k$ is the $n$-th
layer kernel and $\Phi_k^{(n)}$ its mapping function. The
corresponding $(n+1)$-th layer kernel is defined as
\begin{equation}
  \K_k^{(n+1)} (\x,\x')=\sum_{\i\in\I_{\x,k},\j\in\I_{\x',k}}\lambda^{(n+1)}_{\x,\i}\lambda^{(n+1)}_{\x',\j} \prod_{t=1}^k K_{n+1}(\Phi_k^{(n)}(\x_{1:i_t}),\Phi_k^{(n)}(\x'_{1:j_t})), \label{eq:multilayer}
\end{equation}
where $K_{n+1}$ will be defined in the sequel and the choice of weights $\lambda^{(n)}_{\x,\i}$ slightly differs from the single-layer model. 
We choose indeed $\lambda^{(N)}_{\x,\i} = \lambda^{\gap(\i)}$ only for the last layer $N$ of the kernel, which 
depends on the number of gaps in the index set $\i$ but not on the index positions. 
Since~(\ref{eq:multilayer}) involves a kernel $K_{n+1}$ operating on the representation of prefix sequences $\Phi_k^{(n)}(\x_{1:t})$ from layer $n$,
the representation makes sense only if $\Phi_k^{(n)}(\x_{1:t})$ carries mostly local information close to position~$t$.
Otherwise, information from the beginning of the sequence would be overrepresented.
Ideally, we would like the range-dependency of $\Phi_k^{(n)}(\x_{1:t})$ (the size of the window of indices before $t$ that influences the representation, akin to receptive fields in CNNs) to grow with the number of layers in a controllable manner.
This can be achieved by choosing $\lambda^{(n)}_{\x,\i}=\lambda^{|\x|-i_1-k+1}$
for $n < N$, which assigns exponentially more weights to the $k$-mers close to
the end of the sequence.

For the first layer, we recover the single-layer network $\K_k$ defined
in~\eqref{eq:kRKN} by defining $\Phi_k^{(0)}(\x_{1:i_k}) \!=\! \x_{i_k}$ and
$K_1(\x_{i_k},\x'_{j_k}) = e^{\alpha(\langle \x_{i_k},\x'_{j_k}\rangle -1)}$. 
For $n>1$, it remains to define $K_{n+1}$ to be a homogeneous dot-product kernel, as used for
instance in CKNs~\cite{mairal2016end}:
\begin{equation}
    K_{n+1}(\u , \u') = \| \u \|_{\Hcal_n}  \| \u \|_{\Hcal_n} \kappa_n\left (\left\langle \frac{\u}{\|\u\|_{\Hcal_n}},   \frac{\u'}{\|\u'\|_{\Hcal_n}}  \right\rangle_{\Hcal_n}\right)~~~~\text{with}~~~\kappa_n(t) = e^{\alpha_n (t-1)}.\label{eq:dotprod}
\end{equation}
Note that the Gaussian kernel $K_1$ used for 1st layer may also be written as~(\ref{eq:dotprod}) since characters are
normalized. As for CKNs, the goal of homogenization is to prevent norms to grow/vanish exponentially fast with~$n$, while dot-product kernels lend themselves well to neural network interpretations.

As detailed in Appendix~\ref{appendix:multi}, extending the Nystr\"om approximation scheme for the multilayer
construction may be achieved in the same manner as with CKNs---that
is, we learn one approximate embedding $\psi_k^{(n)}$ at each layer,
allowing to replace the inner-products
$\langle \Phi_k^{(n)}(\x_{1:i_t}),\Phi_k^{(n)}(\x'_{1:j_t})\rangle$ by
their approximations
$\langle \psi_k^{(n)}(\x_{1:i_t}),\psi_k^{(n)}(\x'_{1:j_t})\rangle$,
and it is easy to show that the interpretation in terms of RNNs is
still valid since $\K^{(n)}_k$ has the same sum structure
as~(\ref{eq:kRKN}).

\paragraph{Max pooling in RKHS.}
Alignment scores (e.g. Smith-Waterman) in molecular biology rely on a max operation---over
the scores of all possible alignments---to compute similarities
between sequences. However, using max in a string kernel usually
breaks positive definiteness, even though it seems to perform well in practice. To solve such an issue, sum-exponential is used as a proxy in \cite{saigo2004protein}, but it leads to
diagonal dominance issue and makes SVM solvers unable to learn. For
RKN, the sum in~\eqref{eq:kRKNconv} can also be replaced by a max
\begin{equation}
	\K^{\text{max}}_k(\x,\x')= \left\langle \max_{\i\in\I_{\x,k}} \lambda_{\x,\i} \psi_k(\x_\i)  , \max_{\j\in\I_{\x',k}} \lambda_{\x,\j} \psi_k(\x'_\j)\right\rangle, \label{eq:kRKNmax}
\end{equation}
which empirically seems to perform well, but breaks the kernel interpretation, as in~\cite{saigo2004protein}.
The corresponding recursion amounts to replacing all the sum
in~\eqref{eqn:rkn} by a max. 

An alternative way to aggregate local features is the generalized max
pooling (GMP) introduced in \citep{murray2014generalized}, which can
be adapted to the context of RKHSs.  Assuming that before pooling $\x$
is embedded to a set of $N$ local features
$(\varphi_1, \dots, \varphi_N)\in\Hcal^N$, GMP builds a representation
$\varphi^{\gmp}$ whose inner-product with all the local features
$\varphi_i$ is one:
$ \langle \varphi_i, \varphi^{\gmp}
\rangle_{\Hcal}=1,~\text{for}~i=1,\dots,N$. $\varphi^{\gmp}$ coincides
with the regular max when each $\varphi$ is an element of the
canonical basis of a finite representation---\emph{i.e.}, assuming
that at each position, a single feature has value 1 and all others are
$0$.

Since GMP is defined by a set of inner-products constraints, it can be
applied to our approximate kernel embeddings by solving a linear system.
This is compatible with CKN but becomes intractable for RKN which pools across
$|\I_{\x,k}|$ positions. Instead, we heuristically apply GMP over the
set $\psi_k(\x_{1:t})$ for all~$t$ with $\lambda_{\x,\i}=\lambda^{|\x|-i_1-k+1}$, which can be obtained from the RNN
described in Theorem~\ref{thm:kzx}. This amounts to composing GMP with
mean poolings obtained over each prefix of $\x$. We observe that it performs well in our experiments. More details are provided in Appendix~\ref{appendix:gmp}.

%% file: experiments.tex
We evaluate RKN and compare it to typical string kernels and RNN for protein fold recognition.
Pytorch code is provided with the submission and additional details given in Appendix~\ref{appendix:exp}.

\subsection{Protein fold recognition on SCOP 1.67}

Sequencing technologies provide access to gene and, indirectly,
protein sequences for yet poorly studied species. In order to predict
the 3D structure and function from the linear sequence of these
proteins, it is common to search for evolutionary related ones, a
problem known as homology detection. When no evolutionary related
protein with known structure is available, a---more
difficult---alternative is to resort to protein fold recognition. We
evaluate our RKN on such a task, where the objective is to predict
which proteins share a 3D structure with the
query~\citep{rangwala2005profile}.

Here we consider the Structural Classification Of Proteins (SCOP)
version 1.67~\citep{murzin1995scop}. We follow the preprocessing procedures of
\cite{haandstad2007motif} and remove the sequences that are more than
95\% similar, yielding 85 fold recognition tasks. Each positive
training set is then extended with Uniref50 to make the dataset more
balanced, as proposed in~\cite{hochreiter2007fast}. The resulting
dataset can be downloaded from
\url{http://www.bioinf.jku.at/software/LSTM_protein}. The number of
training samples for each task is typically around 9,000 proteins, whose
length varies from tens to thousands of amino-acids. In all our experiments we use logistic loss. We measure
classification performances using auROC and auROC50 scores (area under
the ROC curve and up to 50\% false positives).

For CKN and RKN, we evaluate both one-hot encoding of amino-acids by
20-dimensional binary vectors and an alternative representation
relying on the BLOSUM62 substitution
matrix~\cite{henikoff1992amino}. Specifically in the latter case, we
represent each amino-acid by the centered and normalized vector of its
corresponding substitution probabilities with other amino-acids. The
local alignment kernel~\eqref{eq:LA}, which we include in our
comparison, natively uses BLOSUM62.

\paragraph{Hyperparameters.}
We follow the training procedure of CKN presented
in~\cite{chen2019bio}. Specifically, for each of the $85$ tasks, we
hold out one quarter of the training samples as a validation set, use
it to tune~$\alpha$, gap penalty $\lambda$ and the regularization
parameter $\mu$ in the prediction layer. These parameters are then
fixed across datasets. RKN training also relies on the alternating
strategy used for CKN: we use an Adam algorithm to update anchor
points, and the L-BFGS algorithm to optimize the prediction layer. We
train 100 epochs for each dataset: the initial learning rate for Adam
is fixed to 0.05 and is halved as long as there is no decrease of the
validation loss for 5 successive epochs. We fix $k$ to 10, the number
of anchor points $q$ to 128 and use single layer CKN and RKN
throughout the experiments.

\paragraph{Implementation details for unsupervised models.}
The anchor points for CKN and RKN are learned by k-means on 30,000 extracted $k$-mers from each dataset. The resulting sequence representations are standardized by removing mean and dividing by standard deviation and are used within a logistic regression classifier. $\alpha$ in Gaussian kernel and the parameter $\lambda$ are chosen based on validation loss and are fixed across the datasets. $\mu$ for regularization is chosen by a 5-fold cross validation on each dataset. As before, we fix $k$ to 10 and the number of anchor points $q$ to 1024. Note that the performance could be improved with larger $q$ as observed in \citep{chen2019bio}, at a higher computational cost. 

\paragraph{Comparisons and results.}

The results are shown in Table \ref{tab:scop}. The blosum62 version of
CKN and RKN outperform all other methods. Improvement against the
mismatch and LA kernels is likely caused by end-to-end trained kernel
networks learning a task-specific representation in the form of a
sparse set of motifs, whereas data-independent kernels lead to
learning a dense function over the set of descriptors.  This
difference can have a regularizing effect akin to the $\ell_1$-norm in
the parametric world, by reducing the dimension of the learned linear
function $w$ while retaining relevant features for the prediction
task. GPkernel also learns motifs, but relies on the exact presence of
discrete motifs. Finally, both LSTM and~\cite{lei2017deriving} are
based on RNNs but are outperformed by kernel networks. The latter was
designed and optimized for NLP tasks and yields a $0.4$ auROC50 on this task.

RKNs outperform CKNs, albeit not by a large margin. Interestingly, as
the two kernels only differ by their allowing gaps when comparing
sequences, this results suggests that this aspect is not the most
important for identifying common foldings in a one versus all setting:
as the learned function discriminates on fold from all others, it may
rely on coarser features and not exploit more subtle ones such as
gappy motifs. In particular, the advantage of the LA-kernel against
its mismatch counterpart is more likely caused by other differences
than gap modelling, namely using a max rather than a mean pooling of
$k$-mer similarities across the sequence, and a general substitution
matrix rather than a Dirac function to quantify
mismatches. Consistently, within kernel networks GMP systematically
outperforms mean pooling, while being slightly behind max pooling.

Additional details and results, scatter plots, and pairwise tests
between methods to assess the statistical significance of our
conclusions are provided in Appendix~\ref{appendix:exp}. Note that when $k=14$, 
the auROC and auROC50 further increase to 0.877 and 0.636 respectively.

\begin{table}
	\centering
	\caption{Average auROC and auROC50 for SCOP fold recognition benchmark. LA-kernel uses BLOSUM62 to compare amino acids which is a little different from our encoding approach. Details about pairwise statistical tests between methods can be found in Appendix~E.}\label{tab:scop}
	\begin{tabular}{lccccc}
		\toprule
		Method & pooling & \multicolumn{2}{c}{one-hot} & \multicolumn{2}{c}{BLOSUM62} \\ 
		& & auROC & auROC50 & auROC & auROC50 \\ \midrule
		GPkernel \citep{haandstad2007motif} & & 0.844 & 0.514 & \multirow{3}{*}{--} &\multirow{3}{*}{--} \\ 
		SVM-pairwise \citep{liao2003combining} & & 0.724 & 0.359 & \\
		Mismatch \citep{leslie2004mismatch} & & 0.814 & 0.467 &  \\
		LA-kernel \citep{saigo2004protein} & & -- & -- & 0.834 & 0.504  \\ \midrule
		LSTM \citep{hochreiter2007fast}& & 0.830 & 0.566 & -- & -- \\ \midrule
		CKN-seq  \citep{chen2019bio} & mean & 0.827 & 0.536 & 0.843 & 0.563 \\
		CKN-seq  \citep{chen2019bio} & max & 0.837 & 0.572 & 0.866 & 0.621 \\ 
		CKN-seq & GMP & 0.838 & 0.561 & 0.856 & 0.608 \\ 
		CKN-seq (unsup)\citep{chen2019bio} & mean & 0.804 & 0.493 & 0.827 & 0.548 \\ \midrule
		RKN ($\lambda=0$) & mean & 0.829 & 0.542 & 0.838 &
		0.563 \\ RKN & mean & 0.829 & 0.541 & 0.840 & 0.571 \\
		RKN ($\lambda=0$) & max & 0.840 & 0.575 & 0.862 &
		0.618 \\ RKN & max & 0.844 & \textbf{0.587}
		& \textbf{0.871} & \textbf{0.629} \\
		RKN ($\lambda=0$) & GMP & 0.840 & 0.563 & 0.855 & 0.598 \\
		RKN & GMP & \textbf{0.848} & 0.570 & 0.852 & 0.609 \\ 
		RKN (unsup) & mean & 0.805 & 0.504 & 0.833 & 0.570 \\
		\bottomrule
	\end{tabular}
\end{table}

\subsection{Protein fold classification on SCOP 2.06}
We further benchmark RKN in a fold classification task, following the
protocols used in \cite{10.1093/bioinformatics/btx780}. Specifically,
the training and validation datasets are composed of 14699 and 2013
sequences from SCOP 1.75, belonging to 1195 different folds. The test
set consists of 2533 sequences from SCOP 2.06, after removing the
sequences with similarity greater than 40\% with SCOP 1.75. The input
sequence feature is represented by a vector of 45 dimensions,
consisting of a 20-dimensional one-hot encoding of the sequence, a
20-dimensional position-specific scoring matrix (PSSM) representing
the profile of amino acids, a 3-class secondary structure represented
by a one-hot vector and a 2-class solvent accessibility. We further
normalize each type of the feature vectors to have unit $\ell_2$-norm,
which is done for each sequence position. More dataset details can be
found in \cite{10.1093/bioinformatics/btx780}. We use mean pooling for
both CKN and RKN models, as it is more stable during training for multi-class classification. The other hyperparameters are chosen in the same way as previously. More details about hyperparameter search grid can be found in Appendix~E.

The accuracy results are obtained by averaging 10 different runs and
are shown in Table \ref{tab:scop206}, stratified by prediction
difficulty (family/superfamily/fold, more details can be found
in~\cite{10.1093/bioinformatics/btx780}). By contrast to what we
observed on SCOP 1.67, RKN sometimes yields a large improvement on CKN
for fold classification, especially for detecting distant
homologies. This suggests that accounting for gaps does help in some
fold prediction tasks, at least in a multi-class context where a
single function is learned for each fold.

\begin{table}
\centering
\caption{Classification accuracy for SCOP 2.06. The complete table with error bars can be found in Appendix~E.}\label{tab:scop206}
	\resizebox{\textwidth}{!}{
	\begin{tabular}{lccccccc}
		\toprule
		Method & $\sharp$Params & \multicolumn{3}{c}{Accuracy on SCOP 2.06} & \multicolumn{3}{c}{Level-stratified accuracy (top1/top5/top10)} \\ 
		& & top 1 & top 5 & top 10 & family & superfamily & fold  \\ \midrule
		PSI-BLAST & - & 84.53 & 86.48 & 87.34 & 82.20/84.50/85.30 & 86.90/88.40/89.30 & 18.90/35.10/35.10 \\
		DeepSF & 920k & 73.00  &	 90.25 &	 94.51 & 75.87/91.77/95.14 & 72.23/90.08/94.70 & 51.35/67.57/72.97 \\ 
                CKN (128 filters)      & 211k     & 76.30 & 92.17 & 95.27 & 83.30/94.22/96.00 & 74.03/91.83/95.34 & 43.78/67.03/77.57 \\
           CKN (512 filters)      & 843k     & 84.11 & 94.29 & 96.36 & \textbf{90.24}/\textbf{95.77}/\textbf{97.21} & 82.33/94.20/96.35 & 45.41/69.19/79.73 \\
		\midrule
           RKN (128 filters)               & 211k     & 77.82 & 92.89 & 95.51 & 76.91/93.13/95.70 & 78.56/92.98/95.53 & 60.54/83.78/\textbf{90.54} \\
           RKN (512 filters)               & 843k     & \textbf{85.29} & \textbf{94.95} & \textbf{96.54} & 84.31/94.80/96.74 & \textbf{85.99}/\textbf{95.22}/\textbf{96.60} & \textbf{71.35}/\textbf{84.86}/89.73 \\
		\bottomrule
	\end{tabular}
	}
\end{table}

%% file: appendix.tex
\newcommand\Abf{\mathbf A}
\def\F{\mathbf F}
\newcommand\B{\mathbf B}
\newcommand\U{\mathbf U}
\newcommand\D{\boldsymbol \Delta}
\def\E{\mathbb E}

\section{Nystr\"om Approximation for Single-Layer RKN}\label{appendix:nystrom}
We detail here the Nytr\"om approximation presented in Section~\ref{subsec:nystrom}, which we recall here for a sequence $\x$:
\begin{equation}
    \psi_k(\x) =  K_{ZZ}^{-1/2} \sum_{\i\in\I_{\x,k}} \lambda_{\x,\i}K_Z(\x_\i).\label{eq:psik_app}
    \end{equation}
    Consider then the computation of $\psi_j(\x)$ defined in~(\ref{eq:psik_app}) 
    for $j=1,\ldots,k$ given a set of anchor points $Z_k=\{\z_1,\ldots,\z_q\}$ with the $\z_i$'s in $\Real^{d \times k}$.
    Given the notations introduced in Section~\ref{subsec:nystrom}, we are now in shape to prove Theorem~\ref{thm:kzx}.
\begin{proof}
	The proof is based on Theorem 1 of \cite{lei2017deriving} and definition 2 of \cite{lodhi2002text}. For $\i\in\I_{\x,j}$, let us denote by $\i'=(i_1,\dots,i_{j-1})$ the $j-1$ first entries of $\i$. We first notice that for the Gaussian kernel $K$, we have the following factorization relation for $i=1,\dots,q$
	\begin{align*}
		K(\x_{\i},[\z_{i}]_{1:j})&=e^{\alpha (\langle \x_{\i},[\z_{i}]_{1:j}\rangle - j)} \\
		&=e^{\alpha (\langle \x_{\i'},[\z_{i}]_{1:j-1}\rangle - (j-1))} e^{\alpha (\langle \x_{i_j}, \z_j \rangle -1)} \\
		&=K(\x_{\i'},[\z_{i}]_{1:j-1}) e^{\alpha (\langle \x_{i_j}, \z_j \rangle -1)}.
	\end{align*}
	Thus 
	\begin{equation*}
		K_{Z_j}(\x_{\i})=K_{Z_{j-1}}(\x_{\i'})\odot \b_j[i_j],
	\end{equation*}
	with $\b_j[t]$ defined as in the theorem. 
	
	Let us denote $\sum_{\i\in\I_{\x_{1:t},j}} \lambda_{\x_{1:t},\i}K_{Z_j}(\x_\i)$ by $\tilde{\c}_j[t]$ if $\lambda_{\x,\i}=\lambda^{|\x|-i_1-j+1}$ and by $\tilde{\h}_j[t]$ if $\lambda_{\x,\i}=\lambda^{\gap(\i)}$. We want to prove that $\tilde{\c}_j[t]=\c_j[t]$ and $\tilde{\h}_j[t]=\h_j[t]$. First, it is clear that $\tilde{\c}_j[0]=0$ for any $j$. We show by induction on $j$ that $\tilde{\c}_j[t]=\c_j[t]$. When $j=1$, we have 
	\begin{align*}
		\tilde{\c}_1[t]&=\sum_{1\le i_1\le t } \lambda^{t-i_1} K_{Z_1}(\x_{i_1}) \\
		&=\sum_{1\le i_1 \le t-1} \lambda^{t-i_1} K_{Z_1}(\x_{i_1}) + K_{Z_1}(\x_t), \\
		&=\lambda \tilde{\c}_1[t-1] + \b_1[t].
	\end{align*}
	$\tilde{\c}_1[t]$ and $\c_1[t]$ have the same recursion and initial state thus are identical.
	When $j>1$ and suppose that $\tilde{\c}_{j-1}[t]=\c_{j-1}[t]$, then we have
	\begin{align*}
		\tilde{\c}_j[t]&=\sum_{\i \in\I_{\x_{1:t},j}} \lambda^{t-i_1-j+1} K_{Z_j}(\x_\i),\\
		&= \underbrace{\sum_{\i\in \I_{\x_{1:t-1},j}}  \lambda^{t-i_1-j+1} K_{Z_j}(\x_{\i})}_{i_j<t} + \underbrace{\sum_{\i'\in\I_{\x_{1:t-1},j-1}} \lambda^{(t-1)-s_1-(j-1)+1} K_{Z_{j-1}}(\x_{i'})\odot \b_{j}[t]}_{i_j=t}, \\
		&= \lambda \tilde{\c}_j[t-1] + \tilde{\c}_{j-1}[t]\odot \b_j[t], \\
		&=\lambda \tilde{\c}_j[t-1] + \c_{j-1}[t]\odot \b_j[t].
	\end{align*}
	$\tilde{\c}_j[t]$ and $\c_j[t]$ have the same recursion and initial state. We have thus proved that $\tilde{\c}_j[t]=c_j[t]$. Let us move on for proving $\tilde{\h}_j[t]=\h_j[t]$ by showing that they have the same initial state and recursion. It is straightforward that $\tilde{\h}_j[0]=0$, then for $1\le j\le k$ we have
	\begin{align*}
		\tilde{\h}_j[t]=&\sum_{\i \in\I_{\x_{1:t},j}} \lambda^{i_j-i_1-j+1} K_{Z_j}(\x_i),\\
		=&\sum_{\i\in \I_{\x_{1:t-1},j}}  \lambda^{i_j-i_1-j+1} K_{Z_j}(\x_{\i}) + \sum_{\i'\in\I_{\x_{1:t-1},j-1}} \lambda^{(t-1)-s_1-(j-1)+1} K_{Z_{j-1}}(\x_{i'})\odot \b_{j}[t] \\
		=& \tilde{\h}_j[t-1] + \c_{j-1}[t]\odot \b_j[t].
	\end{align*}
	Therefore $\tilde{\h}_j[t]=\h_j[t]$.
\end{proof}

\section{Back-propagation for Matrix Inverse Square Root}\label{appendix:backprop}
In Section~\ref{subsec:nystrom}, we have described an end-to-end scheme to jointly optimize $Z$ and $\w$. The back-propagation of $Z$ requires computing that of the matrix inverse square root operation as it is involved in the approximate feature map of $\x$ as shown in \eqref{eq:psik_app}. The back-propagation formula is given by the following proposition, which is based on an errata of \citep{mairal2016end} and we include it here for completeness.
\begin{prop}
	Given $\Abf$ a symmetric positive definite matrix in $\mathbb R^{n\times n}$ and the eigencomposition of $\Abf$ is written as $\Abf=\U\D \U^{\top}$ where $\U$ is orthogonal and $\D$ is diagonal with eigenvalues $\delta_1,\dots,\delta_n$. Then
	\begin{equation}
		d(\Abf^{-{\frac{1}{2}}})= - \U (\F \circ (\U^\top d\Abf \U)) \U^\top.
	\end{equation}
\end{prop}
\begin{proof}
	First, let us differentiate with respect to the inverse matrix $\Abf^{-1}$:
\begin{displaymath}
   \Abf^{-1}\Abf = \I \qquad \Longrightarrow \qquad  \Abf^{-1}d\Abf + d(\Abf^{-1}) \Abf = 0 \qquad\Longrightarrow \qquad d(\Abf^{-1}) = - \Abf^{-1} d\Abf \Abf^{-1}.
\end{displaymath}
Then, by applying the same (classical) trick,
\begin{displaymath}
   \Abf^{-{\frac{1}{2}}} \Abf^{-{\frac{1}{2}}} = \Abf^{-1} \qquad \Longrightarrow \qquad d(\Abf^{-{\frac{1}{2}}}) \Abf^{-{\frac{1}{2}}} + \Abf^{-{\frac{1}{2}}}d(\Abf^{-{\frac{1}{2}}})= d(\Abf^{-1}) = - \Abf^{-1} d\Abf \Abf^{-1}.
\end{displaymath}
By multiplying the last relation by $\U^\top$ on the left and by $\U$ on the right.
\begin{displaymath}
     \U^\top d(\Abf^{-{\frac{1}{2}}}) \U \D^{-\frac{1}{2}}  + \D^{-\frac{1}{2}} \U^\top d(\Abf^{-{\frac{1}{2}}}) \U= - \D^{-1} \U^\top d\Abf \U \D^{-1}.
\end{displaymath}
Note that $\D$ is diagonal. By introducing the matrix $\F$ such that $\F_{kl} = \frac{1}{\sqrt{\delta_k}\sqrt{\delta_l}(\sqrt{\delta_k} + \sqrt{\delta_l})}$,
it is then easy to show that 
\begin{displaymath}
     \U^\top d(\Abf^{-{\frac{1}{2}}}) \U= - \F \circ (\U^\top d\Abf \U),
\end{displaymath}
where $\circ$ is the Hadamard product between matrices. Then, we are left with 
\begin{displaymath}
     d(\Abf^{-{\frac{1}{2}}})= - \U (\F \circ (\U^\top d\Abf \U)) \U^\top.
\end{displaymath}
\end{proof}
When doing back-propagation, one is usually interested in computing a
quantity $\bar{\Abf}$ such that given $\bar{\B}$ (with appropriate dimensions),
we have
\begin{displaymath}
    \langle \bar{\B}, d(\Abf^{-{\frac{1}{2}}})  \rangle_F = \langle \bar{\Abf},  d\Abf \rangle_F,
\end{displaymath}
see~\cite{giles2008collected}, for instance.
Here, $\langle, \rangle_F$ denotes the Frobenius inner product. Then, it is easy to show that
\begin{displaymath}
    \bar{\Abf} = - \U (\F \circ ( \U^\top \bar{\B} \U )) \U^\top.
\end{displaymath}

\section{Multilayer Construction of RKN}\label{appendix:multi}
For multilayer RKN, assume that we have defined $\K^{(n)}$ the $n$-th layer kernel. To simplify the notation below, we consider that an input sequence $\x$ is encoded at layer~$n$ as $\x^{(n)}:=(\Phi_k^{(n)}(\x_1), \Phi_k^{(n)}(\x_{1:2}),\dots, \Phi_k^{(n)}(\x) )$ where the feature map at position $t$ is $\x^{(n)}_t=\Phi_k^{(n)}(\x_{1:t})$. The $(n+1)$-layer kernel is defined by induction by
\begin{equation}\label{eq:rkn_multi_app}
	\K_k^{(n+1)} (\x,\x')=\sum_{\i\in\I_{\x,k},\j\in\I_{\x',k}}\lambda^{(n)}_{\x,\i}\lambda^{(n)}_{\x',\j} \prod_{t=1}^k K_{n+1}(\x^{(n)}_{i_t},\x'^{(n)}_{j_t}),
\end{equation}
where $K_{n+1}$ is defined in~(\ref{eq:dotprod}. With the choice of weights described in Section~\ref{subsec:ext},
the construction scheme for an $n$-layer RKN is illustrated in Figure~\ref{fig:rkn_multi}
\begin{figure}
	\centering
	\includegraphics[width=0.7\textwidth]{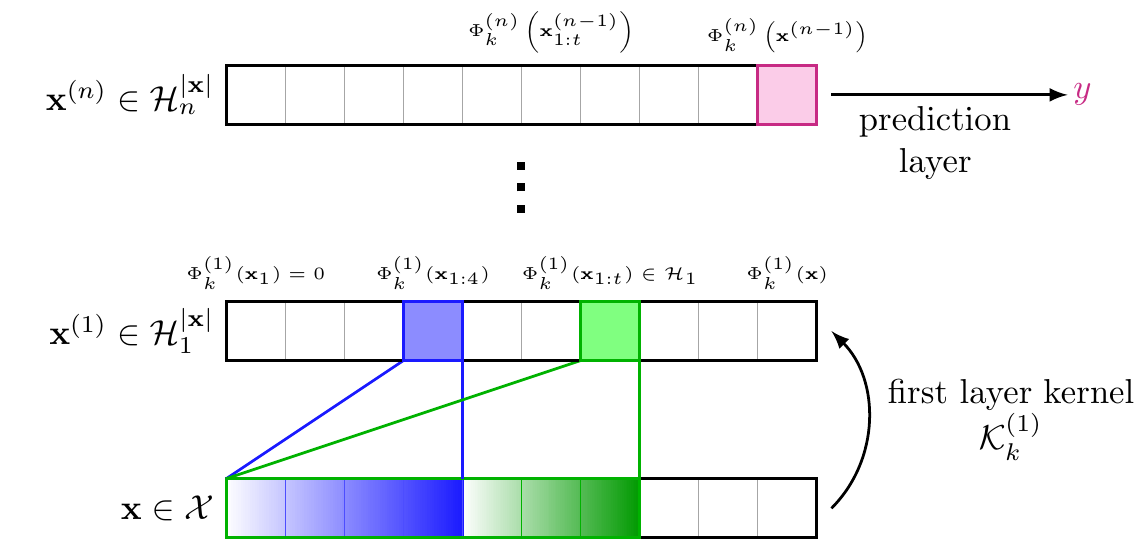}
	\caption{Multilayer construction of RKN: an example with $k=4$.}\label{fig:rkn_multi}
\end{figure}
The Nystr\"om approximation scheme for multilayer RKN is straightforward by inductively applying the Nytr\"om method to the kernels $\K^{(1)},\dots,\K^{(n)}$ from bottom to top layers. Specifically, assume that $\K^{(n)}(\x,\x')$ is approximated by $\langle \psi^{(n)}_k(\x), \psi^{(n)}_k(\x')\rangle_{\R^{q_n}}$ such that the approximate feature map of $\x^{(n)}$ at position $t$ is $\psi^{(n)}_k(\x_{1:t})$. Now Consider a set of anchor points $Z_k=\{\z_1,\ldots,\z_{q_{n+1}}\}$ with the $\z_i$'s in $\Real^{q_n \times k}$ which have unit norm at each column. We use the same notations as in single-layer construction. Then very similar to the single-layer RKN, the embeddings $(\psi^{(n+1)}_j(\x_{1:t}^{(n)}))_{1=1,\dots,k, t=1,\dots,|\x^{(n)}|}$ are given by the following recursion 
\begin{theorem}
	For any $j\in\{1,\dots,k\}$ and
  $t\in\{1,\dots,|\x^{(n)}|\}$,
  \begin{equation*}
   \psi_j^{(n+1)}(\x_{1:t}^{(n)}) =K_{Z_j Z_j}^{-1/2}
\begin{cases}
     \c_j[t] & \text{if } \lambda^{(n)}_{\x,\i}=\lambda^{|\x^{(n)}|-i_1-j+1}, \\
     \h_j[t] & \text{if } \lambda^{(n)}_{\x,\i}=\lambda^{\gap(\i)},
    \end{cases}
  \end{equation*}
 where $c_j[t]$ and $h_j[t]$ form a sequence of vectors in $\Real^{q_{n+1}}$ indexed by $t$ such that $c_j[0]=h_j[0]=0$, and $c_0[t]$ is a vector that contains only ones, while the sequence obeys the recursion
  \begin{equation}\label{eq:rkn_multi_ap}
    \begin{aligned}
      \c_j[t]&=\lambda\c_j[t-1]+\c_{j-1}[t-1] \odot \b_j[t] & 1\le j\le k,\\
      \h_j[t]&=\h_j[t-1]+\c_{j-1}[t-1]\odot \b_j[t]  & 1\le j\le k,
    \end{aligned}
  \end{equation}
  where $\odot$ is the elementwise multiplication operator and $\b_j[t]$ whose entry $i$ is $K_{n+1}(\z^i_j,\x_t^{(n)})=\|\x_t^{(n)}\| \kappa_n \left(\left\langle \z^i_j, \frac{\x_t^{(n)}}{\|\x_t^{(n)}\|} \right\rangle \right)$.
\end{theorem}
\begin{proof}
	The proof can be obtained by that of Theorem~\ref{thm:kzx} by replacing the Gaussian kernel $e^{\alpha (\langle \x_t,\z_j^i )}$ with the kernel $K_{n+1}(\x_t^{(n)},\z_j^i)$.
\end{proof}

\section{Generalized Max Pooling for RKN}\label{appendix:gmp}
Assume that a sequence $\x$ is embedded to $(\varphi_1, \dots,\varphi_n)\in\Hcal^n$ local features, as in Section~\ref{subsec:ext}. Generalized max pooling (GMP) looks for a representation $\varphi^{\gmp}$ such that the inner product between this vector and all the local representations is one:
$	\langle \varphi_i, \varphi^{\gmp} \rangle_{\Hcal}=1,~\text{for }i=1,\dots,n.$
Assuming that each $\varphi_i$ is now represented by a vector $\psi_i$ in $\Real^q$, 
the above problem can be approximately solved by search an embedding vector $\psi^{\gmp}$ in $\R^q$ such that $\langle \psi_i, \psi^{\gmp} \rangle =1$ for $i=1,\dots,n$. In practice, and to prevent ill-conditioned problems, as shown in \citep{murray2014generalized}, it is possible to solve a ridge regression problem:
\begin{equation}
	\psi^{\gmp}=\argmin_{\psi\in\R^q} \| \Psi ^{\top} \psi - \mathbf 1 \|^2 + \gamma \| \psi\|^2,
\end{equation}
where $\Psi=[\psi_1,\dots,\psi_n]\in\R^{q\times n}$ and $\mathbf 1$ denotes the $n$-dimensional vectors with only 1 as entries. The solution is simply given by $\psi^{\gmp}=(\Psi \Psi^{\top}+\gamma I )^{-1}\Psi \mathbf 1$. It requires inverting a $q\times q$ matrix which is usually tractable when the number of anchor points is small. In particular, when $\psi_i=K_{ZZ}^{-1/2} K_Z(\x_i)$ the Nystr\"om approximation of a local feature map, we have $\Psi=K_{ZZ}^{-1/2} K_{ZX}$ with $[K_{ZX}]_{ji}=K(\z_j,\x_i)$ and thus
\begin{equation*}
	\psi^{\gmp}=K_{ZZ}^{\frac{1}{2}} (K_{ZX}K_{ZX}^{\top}+\gamma K_{ZZ} )^{-1} K_{ZX}\mathbf 1.
\end{equation*}

\section{Additional Experimental Material}\label{appendix:exp}
In this section, we provide additional details about experiments and scatter plots with pairwise statistical tests.

\subsection{Protein fold recognition on SCOP 1.67}
\paragraph{Hyperparameter search grids.}
Here, we first provide the grids used for hyperparameter search. In our experiments, we use $\sigma$ instead of $\alpha$ such that $\alpha=1/k \sigma^2$. The search range is specified in Table \ref{tab:hyper}.
\begin{table}[hbtp]
\centering
	\caption{Hyperparameter search range.}\label{tab:hyper}
	\begin{tabular}{lcc}
		\toprule
		hyperparameter & search range \\ \midrule
		$\sigma$ ($\alpha=1/k \sigma^2$) & [0.3;0.4;0.5;0.6] \\
		$\mu$ for mean pooling & [1e-06;1e-05;1e-04] \\
		$\mu$ for max pooling & [0.001;0.01;0.1;1.0] \\
		$\lambda$ & integer multipliers of 0.05 in [0;1]\\
		\bottomrule
	\end{tabular}
\end{table}

\paragraph{Comparison of unsupervised CKNs and RKNs.}
Then, we provide an additional table of results to compare the unsupervised models of CKN and RKN.
In this unsupervised regime, mean pooling perform better than max pooling, which is different than 
what we have observed in the supervised case. RKN tend to work better than CKN, while RKN-sum---that is, using the kernel $\K_{\text{sum}}$ instead of $\K_k$,
works better than RKN.
\begin{table}[hbtp]
	\centering
	\caption{Comparison of unsupervised CKN and RKN with 1024 anchor points.}
	\begin{tabular}{lccccc}
		\toprule
		Method & Pooling & \multicolumn{2}{c}{one-hot} & \multicolumn{2}{c}{BLOSUM62} \\
		& & auROC & auROC50 & auROC & auROC50 \\ \midrule
		CKN & mean & 0.804 & 0.493 & 0.827 & 0.548 \\
		CKN & max & 0.795 & 0.480 & 0.821 & 0.545 \\ 
		RKN ($\lambda=0$) & mean & 0.804 & 0.500 & 0.833 & 0.565 \\
		RKN  & mean & 0.805 & 0.504 & 0.833 & 0.570 \\
		RKN ($\lambda=0$) & max & 0.795 & 0.482 & 0.824 & 0.537 \\ 
		RKN  & max & 0.801 & 0.492 & 0.824 & 0.542 \\
		RKN-sum ($\lambda=0$) & mean & 0.820 & 0.526 & 0.834 & 0.567 \\
		RKN-sum  & mean & 0.821 & 0.527 & 0.834 & 0.565  \\
		RKN-sum ($\lambda=0$) & max & 0.825 & 0.526 & 0.837 & 0.563 \\
		RKN-sum  & max & 0.825 & 0.528 & 0.837 & 0.564 \\
		\bottomrule
	\end{tabular}
\end{table}

\paragraph{Study of filter number $q$ and size $k$.}
Here we use max pooling and fix $\sigma$ to 0.4 and $\lambda$ to 0.1. When $q$ varies $k$ is fixed to 10 and $q$ is fixed to 128 when $k$ varies. We show here the performance of RKN with different choices of $q$ and $k$. The 
gap hyperparameter $\lambda$ is chosen optimally for each $q$ and $k$. The results are shown in Figure~\ref{fig:boxplot}. 

\begin{figure}
	\centering
	\includegraphics[width=0.45\linewidth]{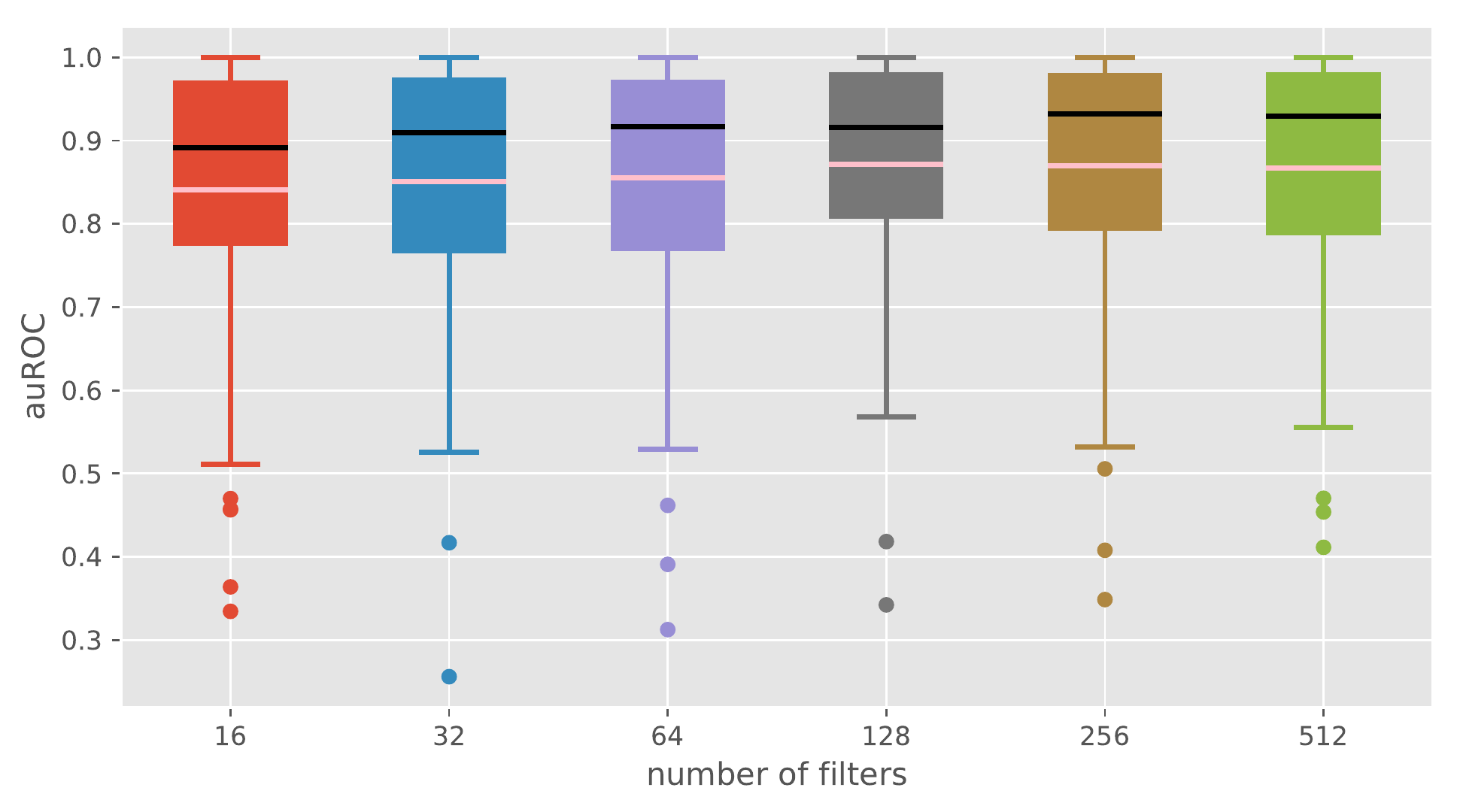}
	\includegraphics[width=0.45\linewidth]{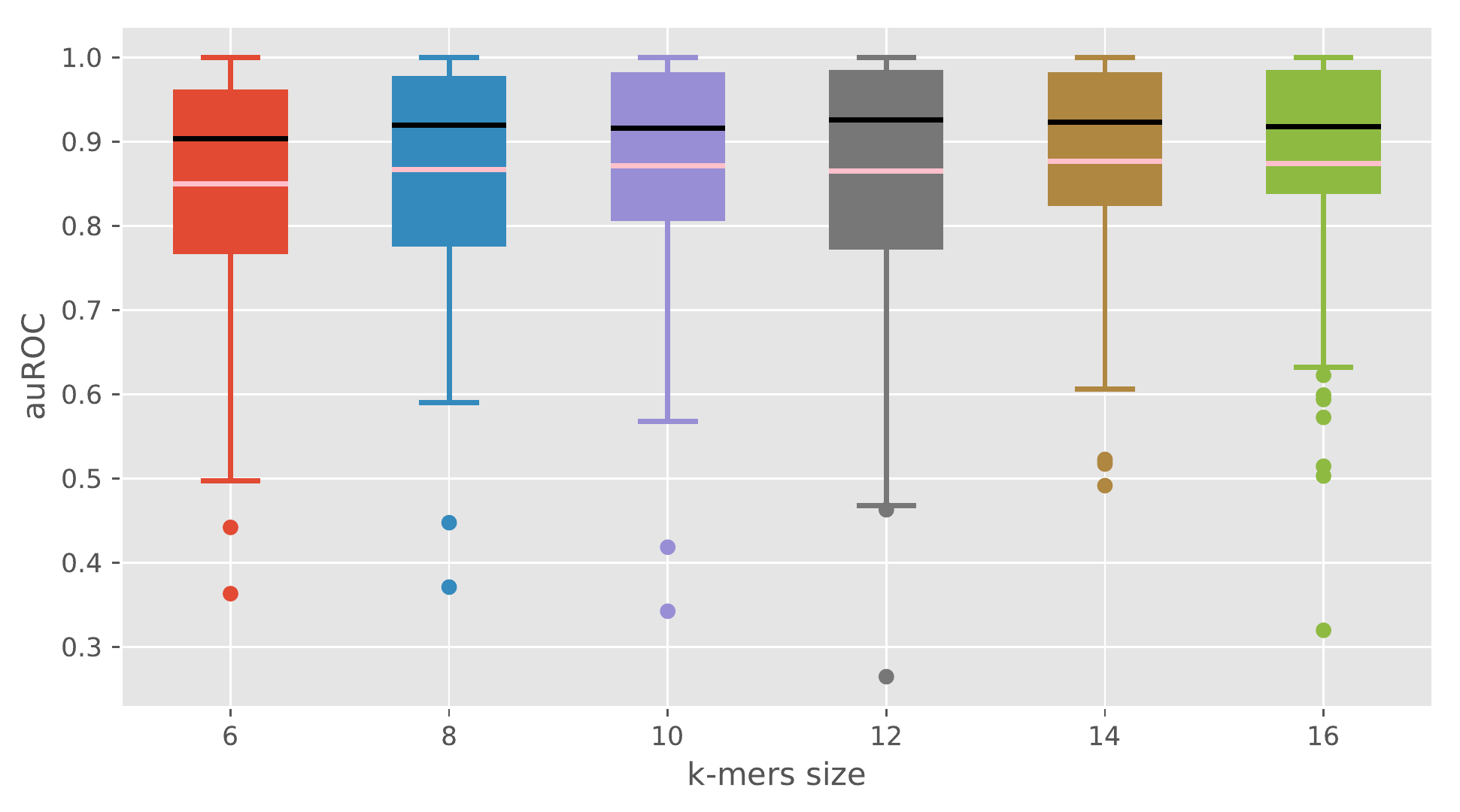}
	\caption{Boxplots when varying filter number $q$ (left) and filter size (right).}\label{fig:boxplot}
\end{figure}

\paragraph{Discussion about complexity.}
Performing backpropgation with our RKN model has the same complexity has a
performing a similar step within a recurrent neural network, up to the
computation of the inverse square root matrix $K_{ZZ}^{-1/2}$, which has
complexity $O(q^3)$. When $q$ is reasonably small $q=128$ in our experiments,
such a complexity is negligible. For instance, one forward pass with a minibatch of $b=128$ sequences of length $m$ yields
a complexity $O(k^2 m b q)$, which can typically be  much greater than $q^3$.

\paragraph{Computing infrastructures.}
Experiments were conduced by using a shared GPU cluster, in large parts build
with Nvidia gamer cards (Titan X, GTX1080TI). About 10 of these GPUs were used
simultaneously to perform the experiments of this paper.

\paragraph{Scatter plots and statistical testing.}
Even though each method was run only one time for each task, the 85 tasks allow
us to conduct statistical testing when comparing two methods.
In Figures~\ref{fig:scat1} and~\ref{fig:scat2}, we provide pairwise comparisons allowing us to 
assess the statistical significance of various conclusions drawn in the paper. We use a
Wilcoxon signed-rank test to provide p-values.
\begin{figure}
\centering
    \includegraphics[width=0.45\linewidth]{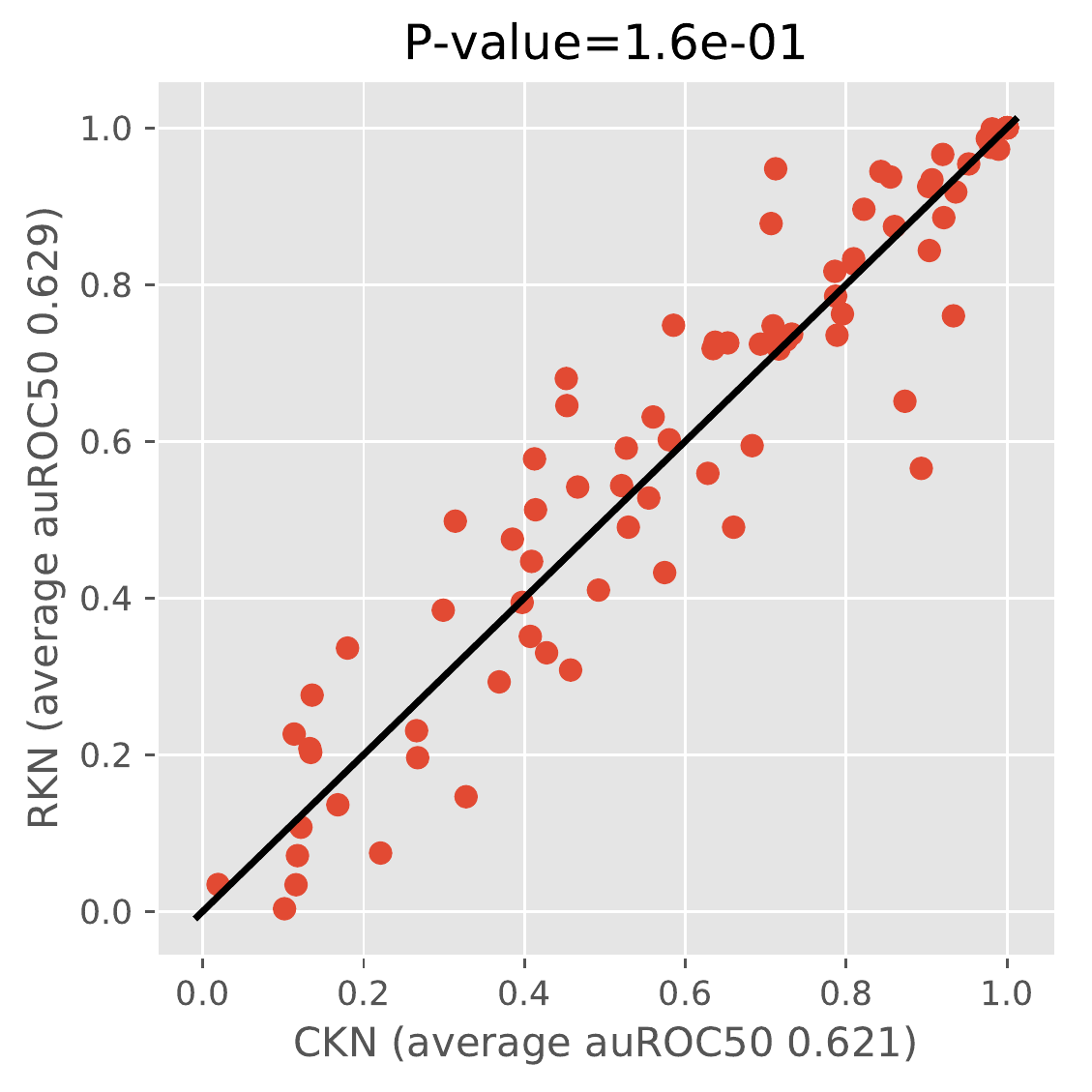}\hfill
    \includegraphics[width=0.45\linewidth]{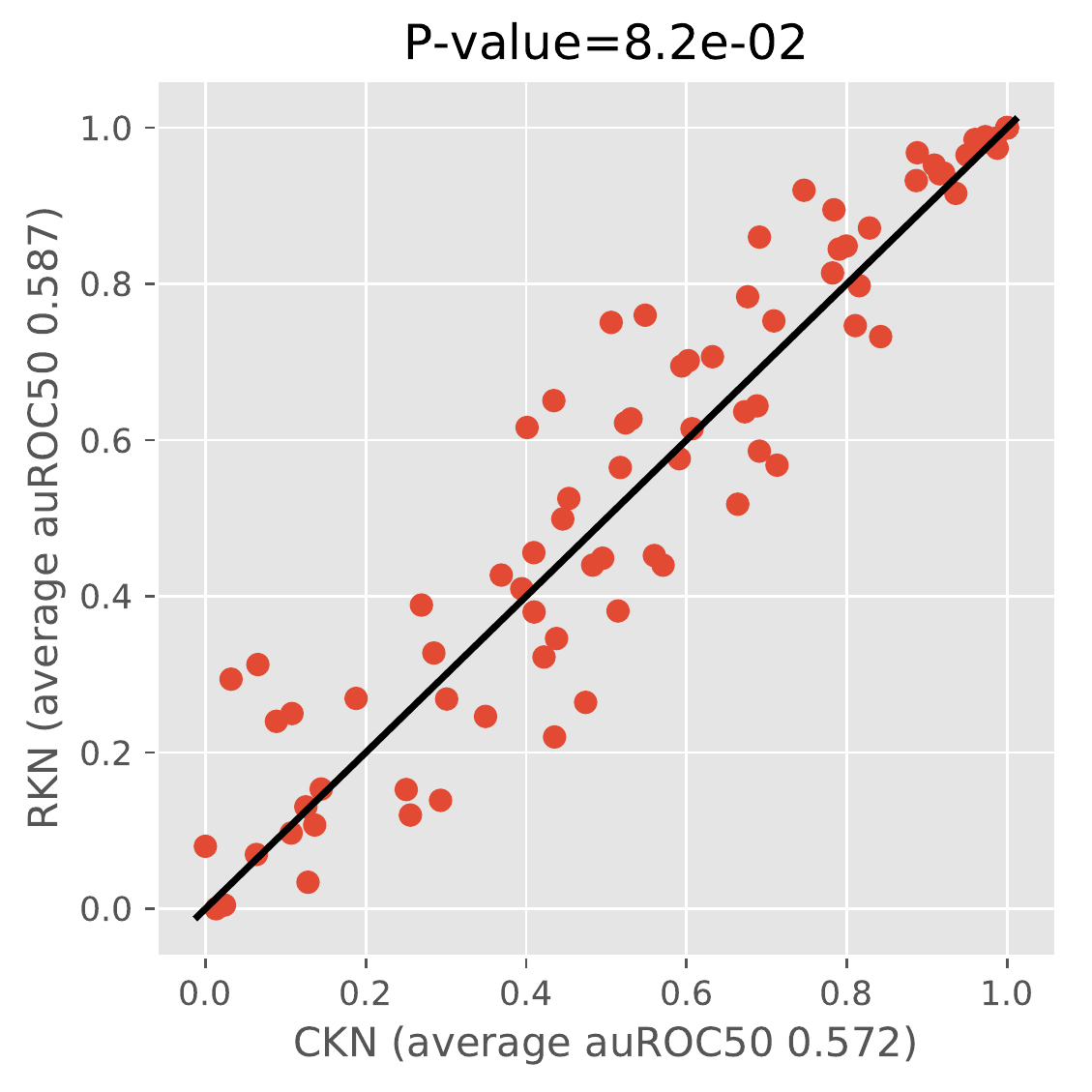}\\
    \includegraphics[width=0.45\linewidth]{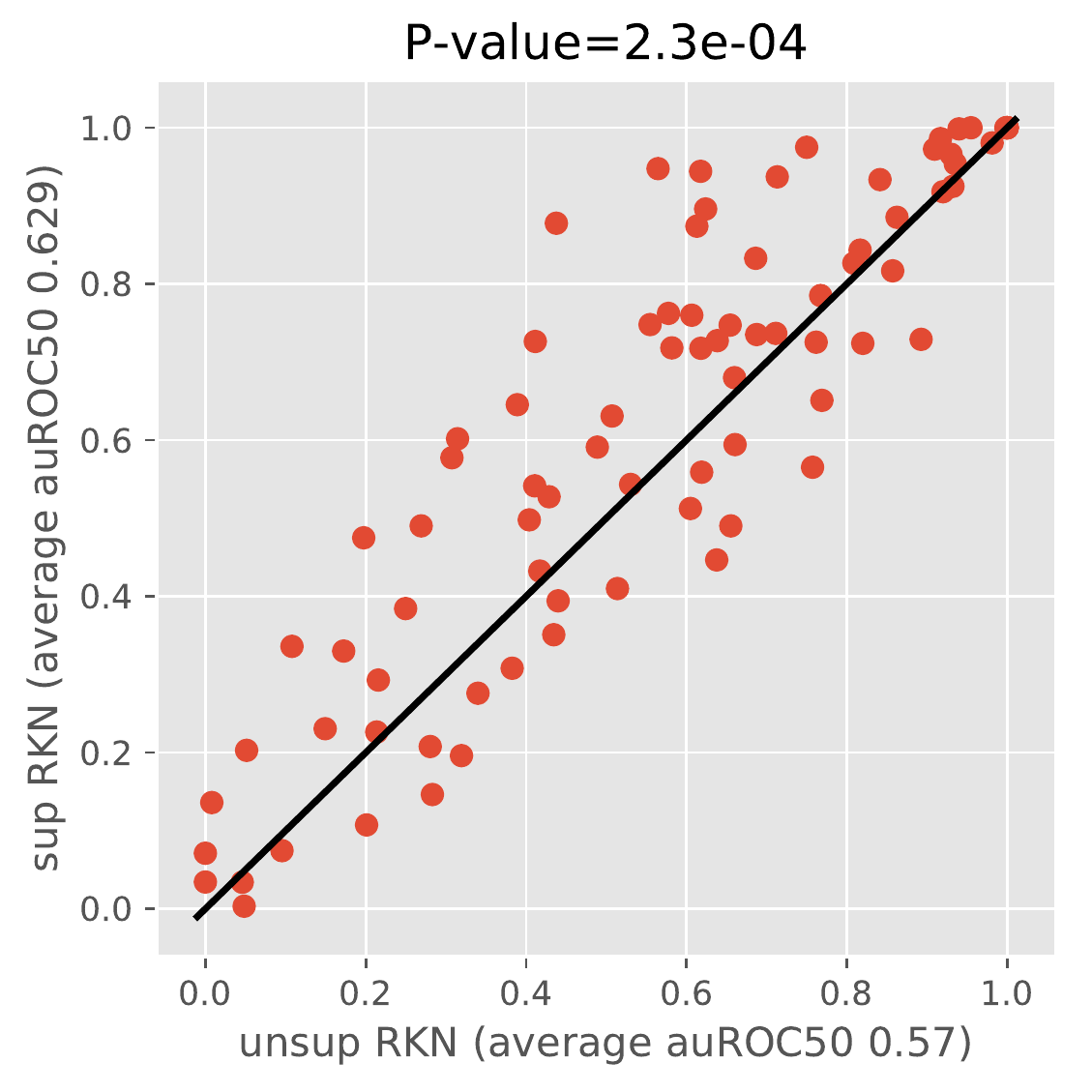}\hfill
    \includegraphics[width=0.45\linewidth]{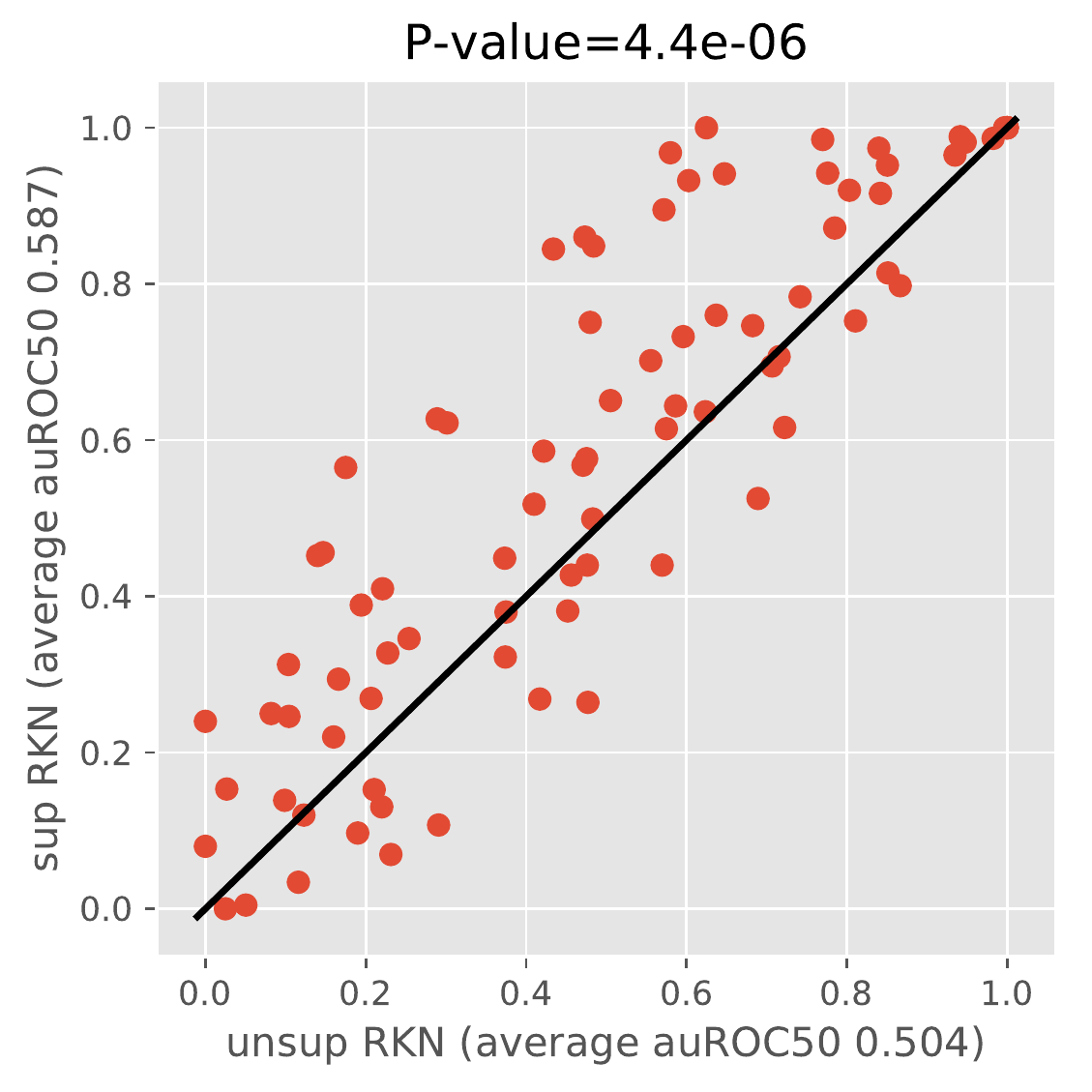}\\
    \includegraphics[width=0.45\linewidth]{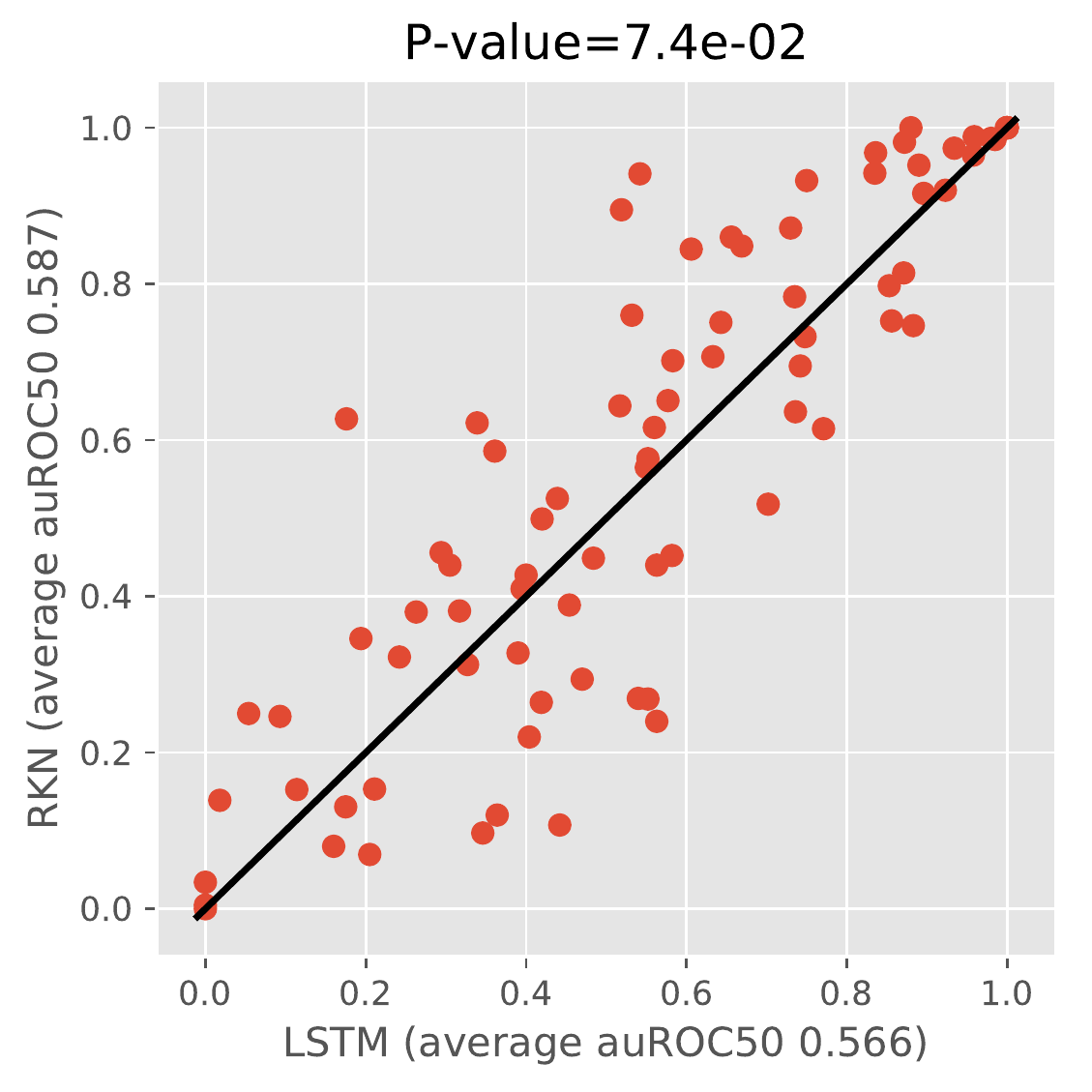}\hfill
\caption{Scatterplots when comparing pairs of methods. In particular, we want to compare RKN vs CKN (top); , RKN vs RKN (unsup) (middle); RKN vs. LSTM (bottom).}\label{fig:scat1}
\end{figure}
\begin{figure}
\centering
    \includegraphics[width=0.45\linewidth]{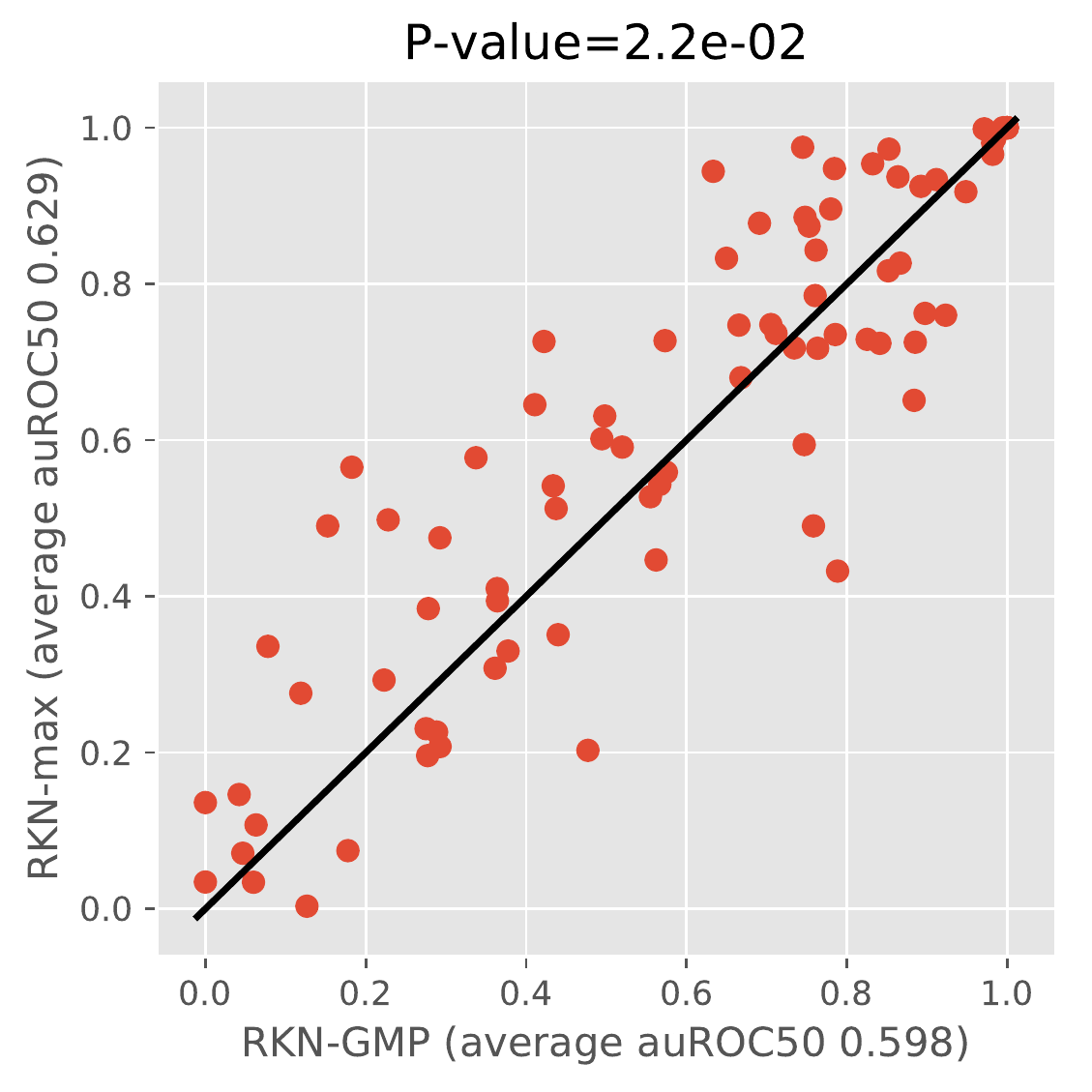}\hfill
    \includegraphics[width=0.45\linewidth]{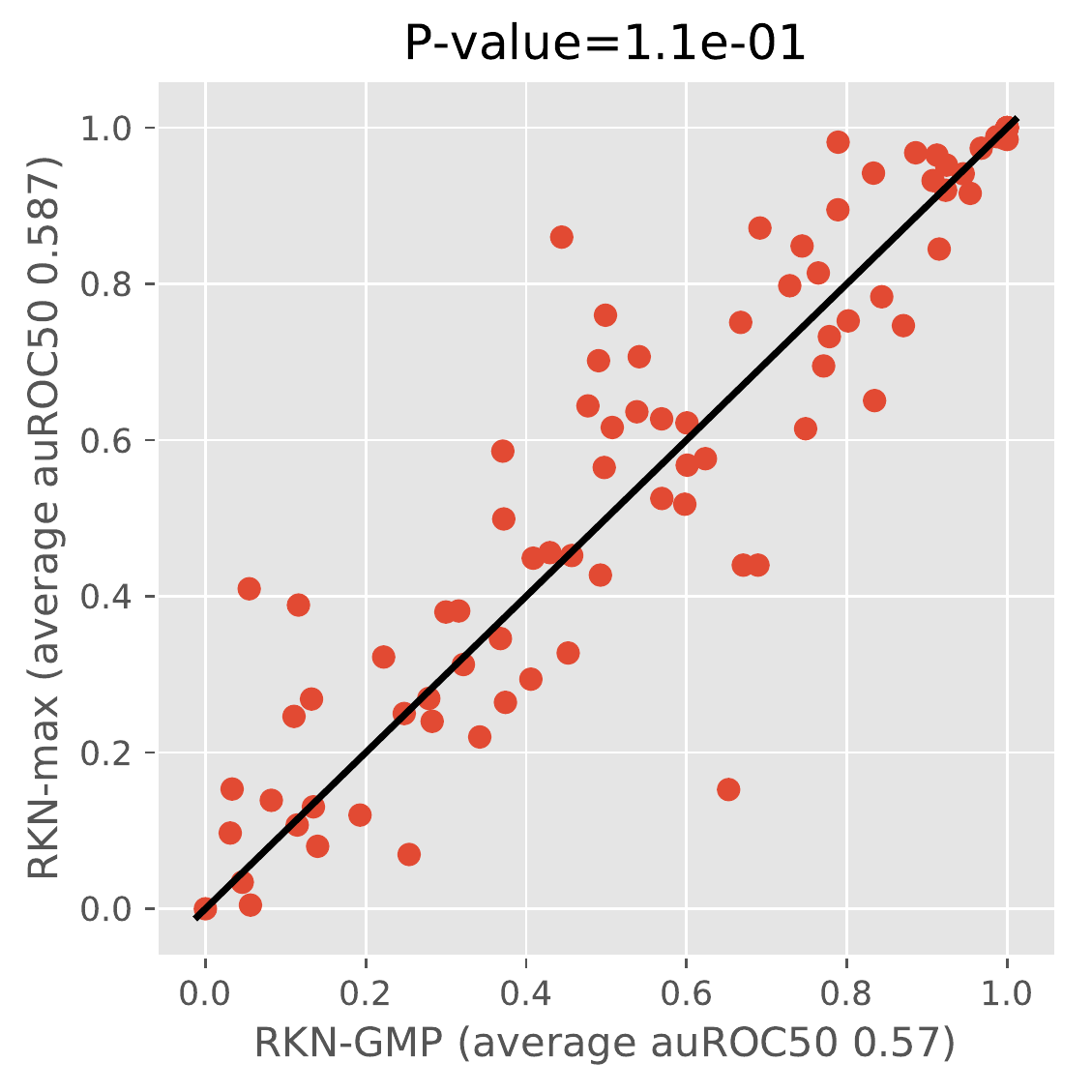}
    \includegraphics[width=0.45\linewidth]{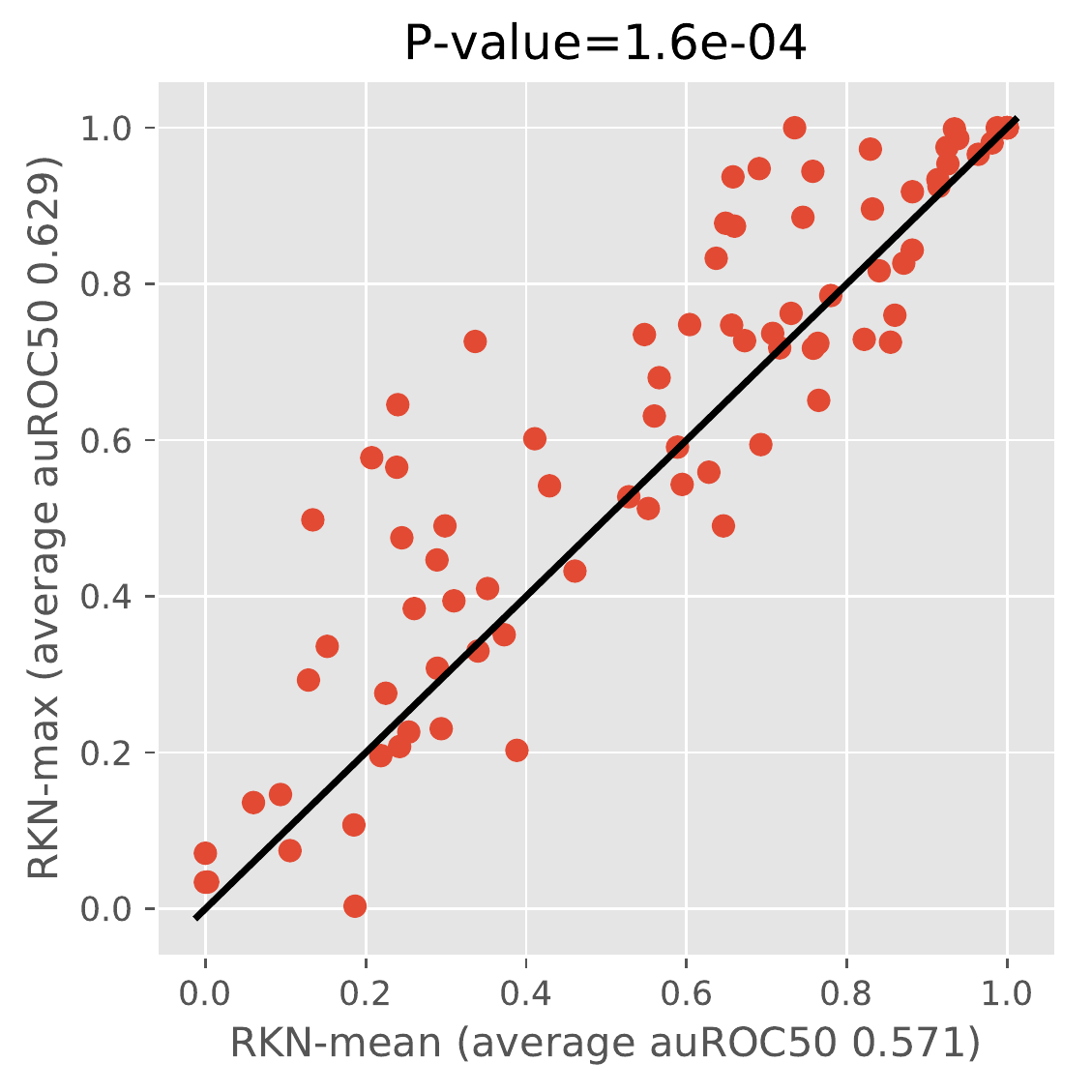}\hfill
    \includegraphics[width=0.45\linewidth]{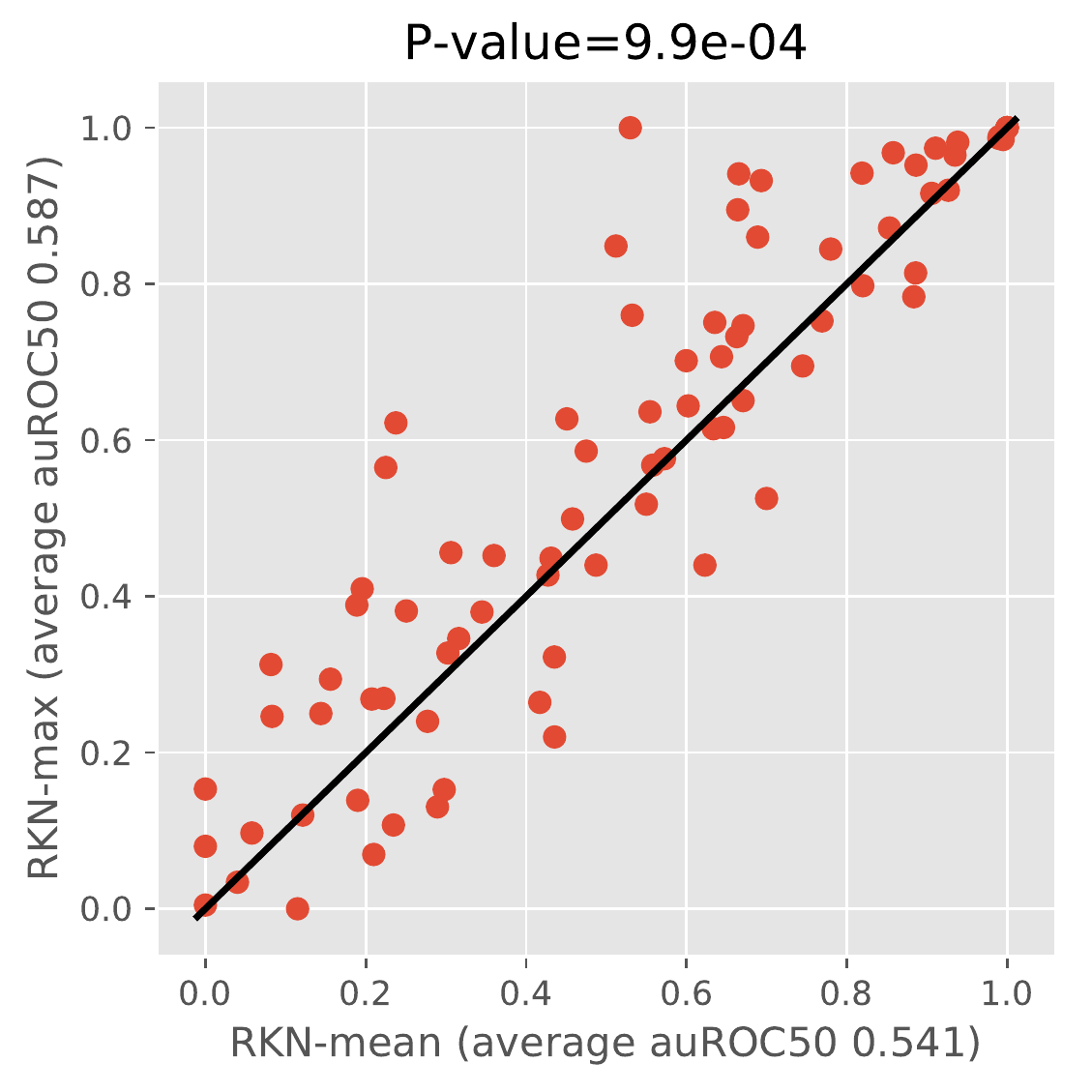}
\caption{Scatterplots when comparing pairs of methods. In particular, we want to compare RKN-gmp vs RKN-max (top); RKN-max vs. RKN-mean (bottom).}\label{fig:scat2}
\end{figure}

\subsection{Protein fold classification on SCOP 2.06}
\paragraph{Hyperparameter search grids.}
We provide the grids used for hyperparameter search, shown in Table~\ref{tab:hyper206}.
\begin{table}[hbtp]
\centering
	\caption{Hyperparameter search range for SCOP 2.06.}\label{tab:hyper206}
	\begin{tabular}{lcc}
		\toprule
		hyperparameter & search range \\ \midrule
		$\sigma$ ($\alpha=1/k \sigma^2$) & [0.3;0.4;0.5;0.6] \\
		$\mu$ & [0.01;0.03;0.1;0.3;1.0;3.0;10.0] \\
		$\lambda$ & integer multipliers of 0.05 in [0;1]\\
		\bottomrule
	\end{tabular}
\end{table}

\paragraph{Complete results with error bars.}
The classification accuracy for CKNs and RKNs on protein fold classification on SCOP 2.06 are obtained by averaging on 10 runs with different seeds. The results are shown in Table~\ref{fig:scop206_errbar} with error bars.
\begin{table}[h]
	\centering
	\caption{Classification accuracy for SCOP 2.06 on all (top) and level-stratified (bottom) test data. For CKNs and RKNs, the results are obtained over 10 different runs.}\label{fig:scop206_errbar}
	\begin{tabular}{lcccc}
		\toprule
		Method & Params & \multicolumn{3}{c}{Accuracy on SCOP 2.06} \\ 
		& & top 1 & top 5 & top 10 \\ \midrule
		PSI-BLAST & - & 84.53 & 86.48 & 87.34 \\
		DeepSF & 920k & 73.00  &	 90.25 &	 94.51 \\
		CKN (128 filters)     & 211k     & 76.30$\pm$0.70 & 92.17$\pm$0.16 & 95.27$\pm$0.17 \\
 CKN (512 filters)     & 843k     & 84.11$\pm$0.16 & 94.29$\pm$0.20 & 96.36$\pm$0.13 \\
		\midrule
 RKN (128 filters)             & 211k     & 77.82$\pm$0.35 & 92.89$\pm$0.19 & 95.51$\pm$0.20  \\
 RKN (512 filters)               & 843k     & 85.29$\pm$0.27 & 94.95$\pm$0.15 & 96.54$\pm$0.12 \\
		\bottomrule
	\end{tabular}
	\resizebox{\textwidth}{!}{
	\begin{tabular}{lcccc}
		\toprule
		Method & \multicolumn{3}{c}{Level-stratified accuracy (top1/top5/top10)} \\ 
		& family & superfamily & fold  \\ \midrule
		PSI-BLAST & 82.20/84.50/85.30 & 86.90/88.40/89.30 & 18.90/35.10/35.10 \\
		DeepSF & 75.87/91.77/95.14 & 72.23/90.08/94.70 & 51.35/67.57/72.97 \\
		CKN (128 filters) & 83.30$\pm$0.78/94.22$\pm$0.25/96.00$\pm$0.26 & 74.03$\pm$0.87/91.83$\pm$0.24/95.34$\pm$0.20 & 43.78$\pm$3.59/67.03$\pm$3.38/77.57$\pm$3.64 \\
 CKN (512 filters)  & 90.24$\pm$0.16/95.77$\pm$0.21/97.21$\pm$0.15 & 82.33$\pm$0.19/94.20$\pm$0.21/96.35$\pm$0.13 & 45.41$\pm$1.62/69.19$\pm$1.79/79.73$\pm$3.68 \\
		\midrule
 RKN (128 filters) & 76.91$\pm$0.87/93.13$\pm$0.17/95.70$\pm$0.37 & 78.56$\pm$0.40/92.98$\pm$0.22/95.53$\pm$0.18 & 60.54$\pm$2.76/83.78$\pm$2.96/90.54$\pm$1.35 \\
 RKN (512 filters) & 84.31$\pm$0.61/94.80$\pm$0.21/96.74$\pm$0.29 & 85.99$\pm$0.30/95.22$\pm$0.16/96.60$\pm$0.12 & 71.35$\pm$1.32/84.86$\pm$2.16/89.73$\pm$1.08 \\
		\bottomrule
	\end{tabular}
	}
\end{table}